\documentclass{article}
\makeatletter
\providecommand{\@trackname}{}
\makeatother

\PassOptionsToPackage{numbers}{natbib}

\usepackage[preprint]{neurips_2026}

\usepackage[bottom]{footmisc}
\makeatletter
\AtBeginDocument{
  \renewcommand{\@notice}{%
    \begingroup
      \renewcommand\thefootnote{}%
      \footnotetext{\@noticestring}%
    \endgroup
  }%
}
\makeatother

\usepackage[utf8]{inputenc} % allow utf-8 input
\usepackage[T1]{fontenc}    % use 8-bit T1 fonts
\usepackage{hyperref}       % hyperlinks
\usepackage{url}            % simple URL typesetting
\usepackage{booktabs}       % professional-quality tables
\usepackage{amsfonts}       % blackboard math symbols
\usepackage{nicefrac}       % compact symbols for 1/2, etc.
\usepackage{microtype}      % microtypography
\usepackage[table]{xcolor}% colors
\usepackage{makecell}
\usepackage{xcolor}
\usepackage{tikz} 

\usepackage{graphicx}
\usepackage{subcaption}
\usepackage{caption}

\usepackage{wrapfig}
\usepackage{hyperref}

\usepackage[normalem]{ulem}

\usepackage{amsmath}
\usepackage{amssymb}
\usepackage{mathtools}
\usepackage{amsthm}
\usepackage{algorithm}
\usepackage{algorithmic}
\usepackage{booktabs}
\usepackage{multirow}
\usepackage{tcolorbox}

\theoremstyle{plain}
\newtheorem{theorem}{Theorem}[section]
\newtheorem{proposition}[theorem]{Proposition}

\theoremstyle{definition}

\theoremstyle{remark}
\newtheorem{remark}[theorem]{Remark}

\usepackage[printonlyused]{acronym}
\acrodef{MELT}{Memory-Efficient Looped Transformer}
\acrodef{CoT}{Chain-of-Thought}
\acrodef{BPTT}{Backpropagation Through Time}
\acrodef{LLMs}{Large Language Models}
\acrodef{KV}{Key-Value}
\acrodef{TRMs}{Tiny Recursive Models}
\acrodef{GQA}{Grouped Query Attention}
\acrodef{CLA}{Cross-Layer Attention}
\acrodef{MLA}{Multi-Head Latent Attention}
\acrodef{RNN}{Recurrent Neural Networks}
\acrodef{MQA}{Multi-Query Attention}
\acrodef{GQA}{Grouped-Query Attention}
\acrodef{CLA}{Cross-Layer Attention}
\acrodef{MLA}{Multi‑Layer Attention}
\acrodef{KD}{Knowledge Distillation}
\usepackage{booktabs}
\usepackage{array}
\usepackage{makecell}
\usepackage{siunitx}
\usepackage{xcolor}
\usepackage[table]{xcolor}

\usepackage{xfrac}
\usepackage{lipsum}

\newcolumntype{G}{>{\columncolor{green!24}}c}

\setcitestyle{numbers,square}

\definecolor{darkgreen}{RGB}{0,100,0}
\definecolor{brown}{RGB}{139,69,19}

\title{Memory-Efficient Looped Transformer: Decoupling Compute from Memory in Looped Language Models}

\author{%
    \vspace{-4pt}
  \textbf{Victor Conchello Vendrell}\thanks{Equal contribution.} \qquad
  \textbf{Arnau Padr\'es Masdemont}\footnotemark[1] \qquad
  \textbf{Niccol\`o Grillo} \\[2pt]
  \textbf{Jordi Ros-Giralt} \qquad
  \textbf{Arash Behboodi} \qquad
  \textbf{Fabio Valerio Massoli} \\[4pt]
  Qualcomm AI Research\thanks{Qualcomm AI Research is an initiative of Qualcomm Technologies, Inc.} \\
  \texttt{\{vconchel, apadres, ngrillo, fmassoli\}@qti.qualcomm.com}
    \vspace{-8pt}
}

\begin{document}

\maketitle

% \begin{abstract}
% \ac{LLMs} capable of ``reasoning'' typically rely on generating extensive intermediate \ac{CoT} tokens, incurring high latency and memory costs. Recurrent architectures, such as LoopLM, offer a parameter-efficient alternative by performing reasoning in the latent space. However, these models suffer from a fundamental trade-off: they either incur linear memory growth ($O(T)$) during training to support \ac{BPTT}, or suffer from catastrophic forgetting and vanishing gradients when constrained.
% In this work, we propose \textbf{\ac{MELT}}, a novel architecture that decouples reasoning depth from memory consumption. Instead of using a standard \ac{KV} cache, \ac{MELT} uses a state matrix with a constant size over loops depth. The state matrix is evolved over time via a learnable gating function update.  
% % This mechanism creates a gradient superhighway, enabling the training of deep reasoning loops ($T > 50$) without vanishing gradients. 
% Empirically, we show that training models obtained by modifying Ouro \citep{zhu2025scaling} to incorporate \ac{MELT} yields superior performance compared to standard LLMs of the same size, while achieving performance comparable to Ouro, confirming \ac{MELT} as an effective constant‑memory reasoning mechanism.
% \end{abstract}

\begin{abstract}
% \vspace{-10pt}
Recurrent LLM architectures have emerged as a promising approach for improving reasoning, as they enable multi-step computation in the embedding space without generating intermediate tokens. Models such as Ouro \citep{zhu2025scalinglatentreasoninglooped} perform reasoning by iteratively updating internal representations while retaining a standard \ac{KV} cache across iterations, causing memory consumption to grow linearly with reasoning depth. Consequently, increasing the number of reasoning iterations can lead to prohibitive memory usage, limiting the practical scalability of such architectures.
In this work, we propose \textbf{\ac{MELT}}, a novel architecture that decouples reasoning depth from memory consumption. Instead of using a standard \ac{KV} cache per layer and loop, \ac{MELT} maintains a single \ac{KV} cache per layer shared across reasoning loops, while retaining per-loop \ac{KV} caches for a small sliding window.
This cache is updated over time via a learnable gating mechanism.
To enable stable and efficient training under this architecture, we propose to train \ac{MELT} using \textit{chunk-wise training} to align training and inference.
Empirically, we show that MELT models fine-tuned from pretrained Ouro parameters maintain the original model performance and \textbf{outperform standard LLMs of comparable size}, while maintaining a memory footprint comparable to those models and dramatically smaller than Ouro's.
Overall, MELT achieves constant-memory iterative reasoning without sacrificing LoopLM performance, using only a lightweight post-training procedure.

\end{abstract}

\vspace{-14pt}
\enlargethispage{12mm}
\acresetall % repeat all acronyms

    % \vspace{-10pt}
\begin{figure}[H]
    \vspace{10pt}
    \centering
    \hfill
    \begin{minipage}{0.49\linewidth}
        \hspace{0pt}
        \includegraphics[width=1.1\linewidth]{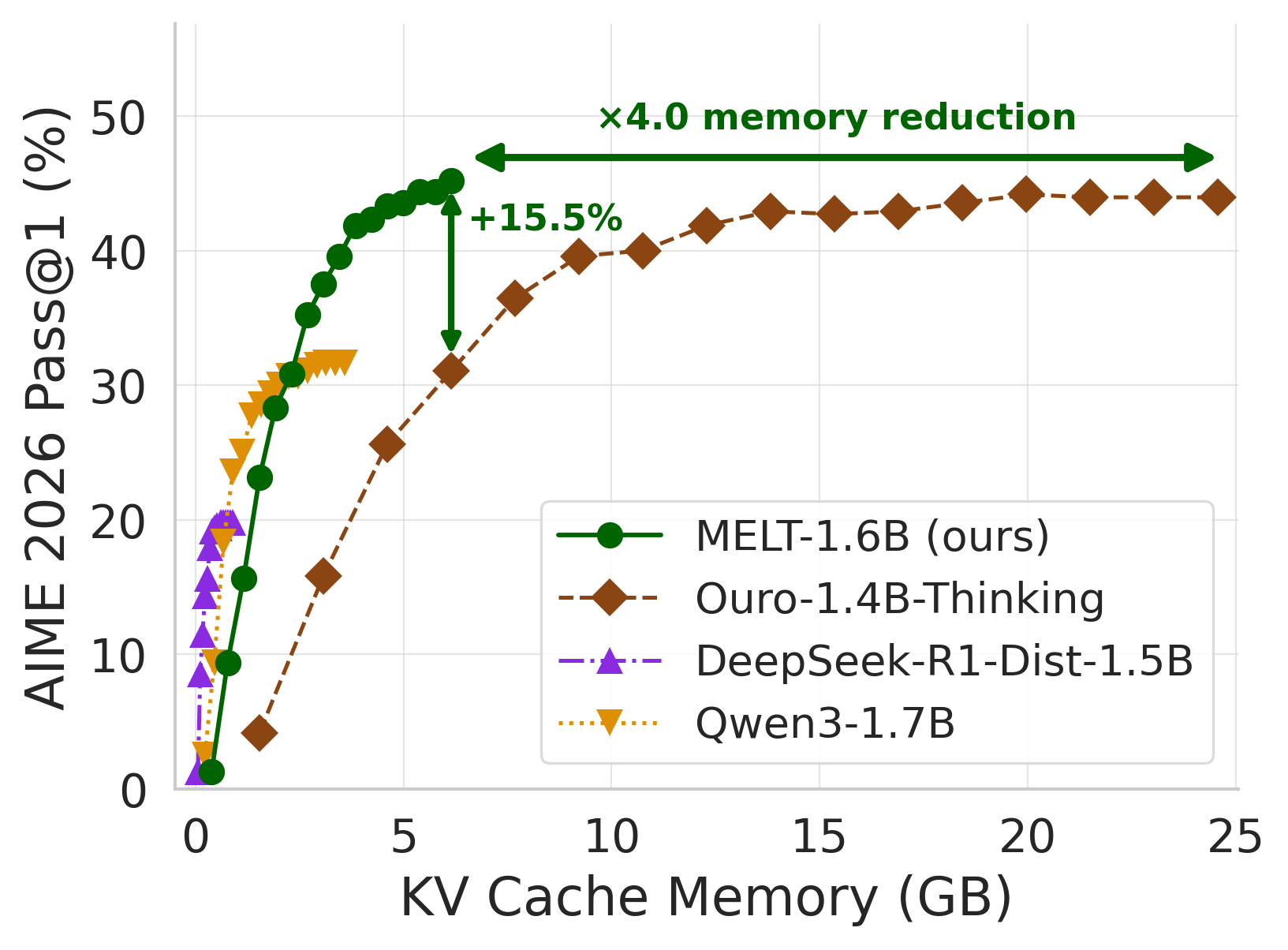}
    \end{minipage}
    \hfill
    \begin{minipage}{0.5\linewidth}
        \centering
        \includegraphics[width=.75\linewidth]{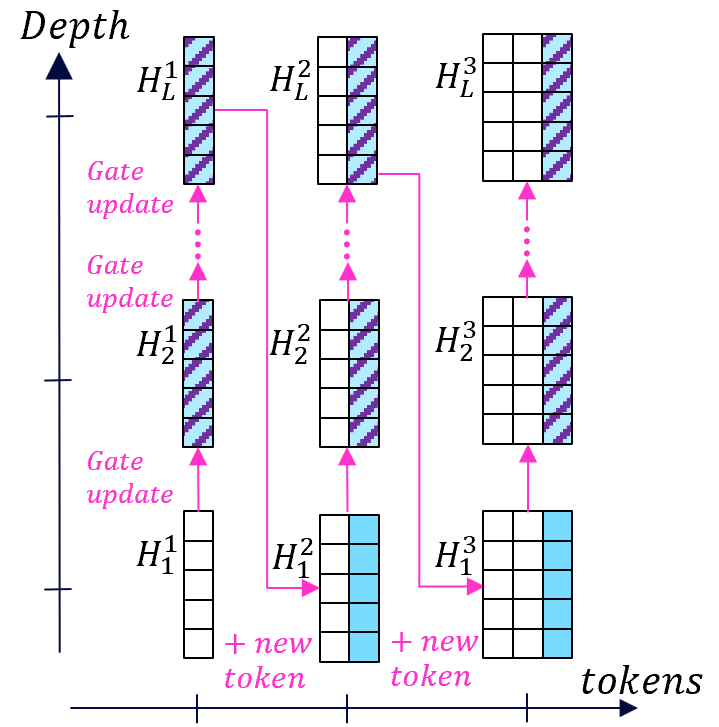}
    \end{minipage}
    \hfill
    \vspace{-10pt}

\begin{minipage}{0.99\linewidth}\vspace{1pt}\end{minipage}

  % \hfill
  % \begin{minipage}{0.53\linewidth} \centering a) AIME26 accuracy versus memory usage \end{minipage}
  % \begin{minipage}{0.45\linewidth} \centering b) High-level illustration of MELT \end{minipage}\hfill
  % \hfill
  % (\textcolor{green}{$\bullet$})
  % ({\color{red}\rotatebox{45}{$\blacksquare$}})
    \caption{
\textbf{MELT Memory Reduction and Architecture Overview. } 
\textbf{Left:} AIME26 accuracy as a function of the total memory footprint. At $\sim 6.5$ GB KV footprint, MELT improves performance by $\sim 15.5\%$ compared to Ouro. MELT achieves $\sim \times 4.0$ memory compared to Ouro at same performance.  
% The light blue region corresponds to 
% MELT achieves superior performance compared to similarly sized non-looped models, while maintaining a comparable memory footprint, 
%only slightly higher due to the absence of MQA.
\textbf{Right:} 
KV cache dynamics. $H_t^n$ denotes the latent state at time step $t$ for token $n$ in a given layer, before being projected into KV. Our single shared KV cache adds one row per token.
 % MELT architecture overview. As in looped transformers, layers are reused across iterations, but the KV cache is updated rather than expanded across loops ({\color{magenta}{$\text{---}$}}).
}
    \label{fig:first_image}
    % \vspace{-50pt}
\end{figure}

\section{Introduction} \label{sec:introduction}

\ac{LLMs} increasingly rely on inference-time compute to improve reasoning, shifting away from purely scaling training-time compute. A dominant approach is \ac{CoT} prompting, where models generate intermediate “thinking” tokens before producing a final answer. While effective, this couples reasoning depth to output length, increasing latency and memory usage. An alternative is latent reasoning, where models perform additional internal computation without generating extra tokens.

A prominent instantiation of latent reasoning is looped transformers, which perform recurrence at the architecture level by repeatedly passing hidden states through the same transformer stack. This approach was first explored in Universal Transformers~\citep{dehghani2019universal} and has recently shown impressive gains with Ouro~\citep{zhu2025scalinglatentreasoninglooped}, demonstrating that looped models can match or surpass transformers nearly twice their size. However, these approaches suffer from a key limitation: memory grows linearly with the number of loops due to \ac{KV} states. To address this, we propose \textbf{\ac{MELT}}, which decouples reasoning depth from memory consumption by maintaining a single KV entry per token and layer, updated across loops via a learnable gating mechanism. This design preserves full attention while keeping memory usage fixed as iterative depth increases.

We demonstrate this approach by training a MELT model initialized from pretrained Ouro~\citep{zhu2025scalinglatentreasoninglooped} weights. Empirically, we show that MELT outperforms similarly sized standard transformers on reasoning benchmarks while preserving the performance of the originating LoopLM, but with dramatically lower memory than looped baselines that retain per-loop KV growth. The main contributions of this paper are:
\begin{itemize}
    \item We introduce \textbf{MELT}, a memory-efficient looped transformer architecture that decouples reasoning depth from memory consumption by sharing a single KV-cache per layer across reasoning loops and updating it with a learnable gating mechanism.
    \item We propose a data-efficient procedure for adapting pretrained LoopLMs to MELT through \textbf{chunk-wise training} and knowledge distillation on all loops.
    \item We empirically show that a MELT model initialized from pretrained Ouro parameters \textbf{outperforms standard LLMs} of comparable size, while matching their memory footprint and using substantially less memory than Ouro.
\end{itemize}

All the code to replicate our experiments and the models will be released soon.

\section{Related work} \label{sec:related}

% SUMMARIZED VERSION

\vspace{-5pt}
This section provides a concise overview of related works, see \autoref{app:related} for an extended version.
\vspace{-5pt}

\paragraph{Looped transformers.}
While \ac{CoT}~\citep{wei2022chain} emphasizes horizontal reasoning, a complementary line of work explores vertical reasoning via recurrent architectures. Early approaches such as HRM and TRM~\citep{wang_hierarchical_2025,jolicoeur2025less}, as well as adaptive-depth methods that dynamically skip or repeat layers~\citep{li2025skip_or_loop_test_time_depth,fu2025think}, highlight the benefits of iterative computation. More broadly, looped transformers have emerged as a strong architectural paradigm, outperforming similarly sized vanilla transformers on multi-hop reasoning, length generalization, and algorithmic tasks~\citep{saunshi2025latent_thoughts_looped_transformers,kohli2026loopthinkgeneralize,fan2025looped_length_generalization,yang2024looped_learning_algorithms}. Despite classical optimization challenges such as instability and vanishing gradients~\citep{dehghani2019universal}, recent work demonstrates stable training at scale~\citep{zhu2025scalinglatentreasoninglooped,geiping2025recurrent_depth_latent_reasoning, prairie2026parcaescalinglawsstable} across different designs, including fully looped stacks and middle-cycle architectures~\citep{geiping2025recurrent_depth_latent_reasoning,zeitoun2026hyperlooptransformers}. These results establish looped transformers as a promising direction for scalable reasoning through iterative compute.

\paragraph{KV cache compression and vertical sharing.}
Efficient KV cache management is critical in looped and long-context models, where memory typically scales with effective depth. Prior work has explored redundancy across heads, layers, and recurrence steps, including MQA/GQA for head sharing~\citep{shazeer2019mqa,ainslie2023gqa} and cross-layer reuse methods such as CLA and MLA~\citep{brandon2024cross,deepseek2024}. In looped transformers, several approaches reduce KV growth by selectively reusing or compressing cached states, including hybrid global–local attention~\citep{wu2025parallellooptransformerefficient}, recursion-aware caching and sharing~\citep{bae2025mixtureofrecursionslearningdynamicrecursive}, and untrained reuse across loops~\citep{geiping2025recurrent_depth_latent_reasoning,zhu2025scalinglatentreasoninglooped}. While some of these methods can reduce memory costs in constrained settings, their effectiveness remains limited on long, complex reasoning tasks, where they often lead to performance degradation when applied to stronger models and longer reasoning traces (see Appendix~\ref{app:untrained_ablations}).

\paragraph{Training transitions and representation-level distillation.}
Adapting pretrained models to new architectures requires gradual transitions to avoid destabilization. Our approach is most closely related to progressive growing~\citep{karras2018progressivegrowinggansimproved}, which interpolates between existing and newly introduced components, and to subsequent work on gradual training and adaptation~\citep{chen2026progressiveresidualwarmuplanguage,li2017learningforgetting} as well as architectural modification in LLMs~\citep{cheng2026attentioneditingversatileframework,komatsuzaki2023sparseupcyclingtrainingmixtureofexperts}. Complementarily, \ac{KD}~\citep{hinton_distilling_2015} has been used to stabilize model adaptation, with prior work showing that aligning intermediate representations improves transfer and robustness~\citep{aguilar_knowledge_2020,chen_cross-layer_2021}. This has proven effective in LLMs, where layer-wise supervision enables compact models~\citep{muralidharan_compact_2024} and strict activation matching mitigates representation drift in complex reasoning settings~\citep{hao_token_2025,fang_knowledge_2026}.

\section{Memory-Efficient Looped Transformer} \label{sec:method}

\subsection{Preliminaries}

\paragraph{Notation.}
Throughout the paper, we use the following notation.  
The model has $N$ \emph{layers}, each a distinct transformer block with its own parameters, and uses a \textit{hidden dimension} $d$ for internal representations. The \emph{sequence length}, $L$, corresponds to the number of tokens in the input. The \emph{reasoning depth} or \emph{time index}, $T$, refers to the number of \textit{reasoning loops} or \textit{time steps} applied to a single token.

\paragraph{LoopLM architecture.}
We adopt the LoopLM architecture~\citep{zhu2025scalinglatentreasoninglooped} for causal sequence modeling, following the formulation used in prior looped‑reasoning models. This design increases per‑token computation without expanding the parameter count. Let $\mathrm{emb}(\cdot) : \mathbb{R}^{|V|} \to \mathbb{R}^{d}$ denote the token embedding map, $\mathcal{T}_{\theta}(\cdot) : \mathbb{R}^{L \times d} \to \mathbb{R}^{L \times d}$ a causal Transformer layer with parameters $\theta$, and $\mathrm{lmhead}(\cdot) : \mathbb{R}^{d} \to \mathbb{R}^{|V|}$ the output projection. A standard (non‑looped) language model composes $N$ layers as $\mathcal{M} = \mathcal{T}_{\theta_N} \circ \cdots \circ \mathcal{T}_{\theta_1}$. In the looped setting, this stack is applied repeatedly for $T$ iterations, so the forward pass becomes:
\[
F(\cdot) = \mathrm{lmhead} \circ \mathcal{M} \circ \mathrm{emb}(\cdot)
\;\rightarrow\;
F^{(T)}(\cdot)
=
\mathrm{lmhead} \circ
\underbrace{\mathcal{M} \circ \cdots \circ \mathcal{M}}_{\text{$T$ iterations}}
\circ \mathrm{emb}(\cdot).
\]

% Let $\mathrm{emb}(\cdot) : \mathbb{R}^{|V|} \to \mathbb{R}^{d}$ denote the token embedding map, $\mathcal{T}_{\theta}(\cdot) : \mathbb{R}^{L \times d} \to \mathbb{R}^{L \times d}$ a causal Transformer layer with parameters $\theta$, and $\mathrm{lmhead}(\cdot) : \mathbb{R}^{d} \to \mathbb{R}^{|V|}$ the output projection. A standard (non‑looped) model stacks $N$ layers; using $\circ$ for composition, a forward pass is
% \[
% F(\cdot) = \mathrm{lmhead} \circ \mathcal{M} \circ \mathrm{emb}(\cdot), \qquad 
% \mathcal{M}(\cdot) = \mathcal{T}_{\theta_N} \circ \cdots \circ \mathcal{T}_{\theta_1}(\cdot).
% \]
% In the looped formulation, the same depth‑$N$ stack is applied for $T$ iterations before the output head:
% \[
% \label{eq:looped-forward}
% F^{(T)}(\cdot)
% = \mathrm{lmhead} \circ 
% \underbrace{\mathcal{M} \circ \cdots \circ \mathcal{M}}_{\text{$T$ iterations}} 
% \circ \, \mathrm{emb}(\cdot).
% \]

\vspace{-10pt}
\subsection{Architecture} \label{sec:arch}

\begin{figure}[!t]
\centering
    \centering
    \includegraphics[width=0.7\linewidth]{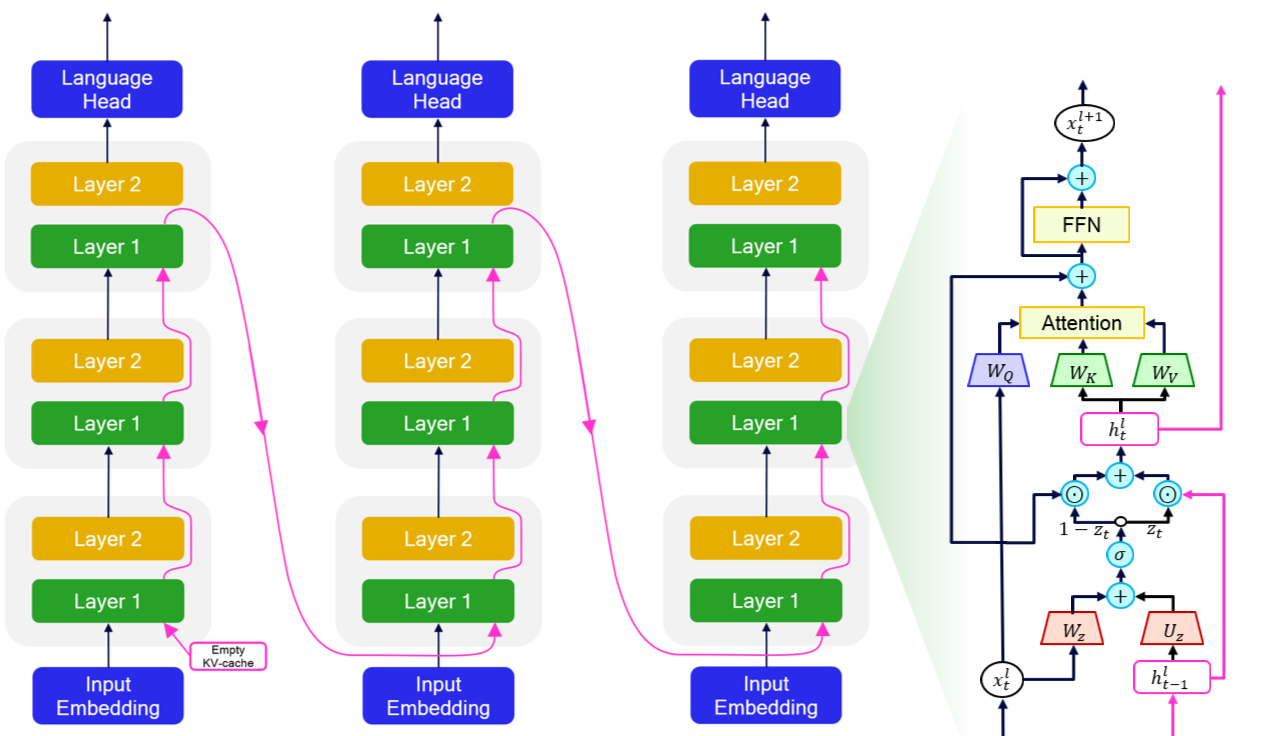}
    % \includegraphics[width=\linewidth]{figures/method-training.png}
  % \begin{minipage}{0.47\linewidth} \centering a) Full MELT model \end{minipage}\hfill
  % \begin{minipage}{0.21\linewidth} \centering 
  % \end{minipage}\hfill
  % \begin{minipage}{0.3\linewidth} \centering b) Layer architecture \end{minipage}
  \caption{
    \textbf{Visualization of the MELT architecture and its KV cache dynamics.} The pink arrows highlight the flow of the KV cache for layer 1. The schematic emphasizes the shared-cache path and abstracts away the small local sliding window over the last $W$ generated tokens.
    % \textbf{(a)} Example with 2 layers and 3 loops. As in looped transformers, layers are reused across iterations, but the KV cache is updated rather than expanded across loops. 
    % \textbf{(b)} Each layer follows a standard transformer structure, augmented with a latent state update used to compute the KV representations. 
    % \textbf{(c)} $H_t^n$ denotes the latent state at time step $t$ for token $n$ in a given layer, before being projected into KV. A single shared KV cache adds one row per token and updates it across loops via the gating mechanism.
  }
  \vspace{-10pt}
  \label{fig:method_inference}
\end{figure}

There are three key differences that separate our architecture from LoopLM: \vspace{-5pt}
\begin{itemize}
    \item The per-layer \ac{KV} cache has a fixed size independent of the reasoning depth, except for a small sliding window over the most recent tokens. Consequently, the total cache scales as $\mathcal{M}_{\text{MELT}} \propto \mathcal{O}(N \times L + W \times L \times T)$, where $W \ll N$ is the window size, compared to $\mathcal{M}_{\text{LoopLM}} \propto \mathcal{O}(N \times L \times T)$.
    \item Instead of appending a new state at every loop step, each loop \emph{updates} the cached state of the token. A new state is added only at the first time step, and after all iterations these updated states are passed to subsequent tokens.
    \item Our gating mechanism enables each token, at each time step, to attend to keys and values that integrate information across \emph{all} time steps of preceding tokens, rather than only the current step.
\end{itemize}
\vspace{-10pt}

\paragraph{Preserving query-key alignment and memory dynamics.}
A key design choice in \ac{MELT} is to maintain a separate latent state $h$ that evolves across iterations, from which keys and values are derived through learned projections ($W_K, W_V$), rather than directly updating the \ac{KV} cache at each loop step. This choice is motivated by preserving semantic integrity, decoupling memory updates from attention retrieval, and maintaining query–key alignment. By evolving a latent state and projecting it into $K, V$ space, we preserve alignment across recurrent updates while separating memory dynamics from retrieval.

This design also leads to a fundamentally different memory behavior. Standard looped transformers follow an append-only strategy, where the per-layer \ac{KV} cache grows linearly with both sequence length $L$ and reasoning depth $T$, i.e., \mbox{$\mathcal{M}^{(l)}_{\text{LoopLM}} \propto \mathcal{O}(L \times T)$}, resulting in prohibitively large memory overhead for deep reasoning. In contrast, \ac{MELT} maintains a latent state $h^{(l)}_t$ with size independent from depth, yielding \mbox{$\mathcal{M}^{(l)}_{\text{MELT}} \propto \mathcal{O}(L)$}. The latent state is updated via a learnable gated momentum mechanism:
\begin{equation} \label{eqn:gated_update}
\begin{aligned}
    z^{(l)}_t &= \sigma\left(x^{(l)}_t W^{(l)}_z + h^{(l)}_{t-1} U^{(l)}_z + b^{(l)}_z \right), \\
    h^{(l)}_t &= z^{(l)}_t \odot h^{(l)}_{t-1} + (1 - z^{(l)}_t) \odot x^{(l)}_t    
\end{aligned}
\end{equation}
where $x^{(l)}_t$ is the hidden state and $z^{(l)}_t$ the gating function. Outside the bounded sliding window, this reduces the depth-wise memory complexity to $\mathcal{O}(1)$ per layer, effectively recovering the footprint of non-looped transformers. As a result, the burden of retaining information shifts from explicit storage (KV cache) to the learned gating dynamics, which determine what information is preserved or overwritten over time.

\paragraph{Integration into the Transformer.}
This latent state is then used to generate the key and value representations for the current token,
\[
k^{(l)}_t = h^{(l)}_t W^{(l)}_{K}
\qquad
v^{(l)}_t = h^{(l)}_t W^{(l)}_{V}
\]
where $W^{(l)}_{K}, W^{(l)}_{V} \in \mathbb{R}^{d \times d}$ are learned projection matrices. The resulting $k^{(l)}_t$ and $v^{(l)}_t$ are appended to the KV-cache produced by earlier tokens at the same layer,
\[
K^{(l)}_{t} = \bigl[K^{(l)}, k^{(l)}_t \bigr],
\qquad
V^{(l)}_{t} = \bigl[V^{(l)}, v^{(l)}_t\bigr],
\]
which is then consumed by the attention mechanism to compute the updated hidden state $x^{(l+1)}_t$
\[
    x^{(l)}_{attn}=Attn^{(l)}(q^{(l)}, K^{(l)}_t, V^{(l)}_t) + x^{(l)}_t \quad\quad
    x^{(l+1)}_t = FFN^{(l)}(x^{(l)}_{attn}) + x^{(l)}_{attn}
\]

\paragraph{Per-loop sliding-window KV cache.}
\ac{MELT} also keeps a sliding window over the last $W$ generated tokens. For tokens outside this window, the keys and values used by attention are computed from the final-loop latent state $h_T^{(l)}$. For tokens inside the window, attention at loop $t$ instead computes keys and values from the current-loop latent state $h_t^{(l)}$. Therefore, only the last $W$ generated tokens retain loop-specific keys and values, yielding a per-layer \ac{KV} cache size of \mbox{$\mathcal{M}^{(l)}_{\text{MELT}} \propto \mathcal{O}(L \times T)$} for this window. During prompt prefilling, each loop constructs its cache from previous-token hidden states at the same loop, preserving parallel processing across prompt tokens.

An overview of the MELT architecture is shown in \autoref{fig:method_inference}. Further theoretical analysis and analysis of gradient flow and stability is provided in \autoref{app:theory}.

\subsection{Training details}\label{sec:training_details}

A key challenge in training MELT arises from its KV-cache computation, which introduces a sequential dependency across tokens: the KV cache for token $t{+}1$ can only be computed after completing the forward pass for token $t$. This contrasts with standard transformers (and Ouro), where KV caches depend only on per-layer activations, enabling parallel token processing during SFT. While fully autoregressive training would respect this dependency, it is prohibitively slow, whereas bypassing the final reasoning loop restores parallelism but introduces a mismatch with inference dynamics.

\paragraph{Chunk-wise training.}
To balance efficiency and fidelity, we propose \textit{chunk-wise training}, illustrated in \autoref{fig:method_training}. Sequences are split into fixed-length chunks processed sequentially, while computations within each chunk are performed in parallel using the current loop’s latent state. Across chunks, the full computation is completed and the final latent state is propagated, better approximating autoregressive inference. The chunk size controls the fidelity–efficiency trade-off: smaller chunks more closely match inference at the cost of throughput, while larger chunks improve efficiency but introduce greater deviation.

\paragraph{Knowledge distillation on all loops.} 
To preserve alignment with the pretrained model, we apply Knowledge Distillation~\citep{hinton_distilling_2015} using the initial LoopLM as teacher, applying supervision at all reasoning loops. This denser signal improves convergence and stabilizes training.

\section{Experimental results} \label{sec:experiments}

\subsection{Experimental setup} \label{sec:exp-setup}

We initialize our models, MELT-1.6B and MELT-3.1B, using the pretrained weights of Ouro-1.4B-Thinking and Ouro-2.6B-Thinking respectively~\citep{zhu2025scalinglatentreasoninglooped}, except for the new gating parameters, which are initialized randomly. 
Because MELT modifies the \ac{KV} cache structure and introduces randomly initialized gating parameters, this hybrid initialization leads to initially incoherent outputs. To address this, we fine‑tune the full model using \textit{chunk‑wise} training, as described in \autoref{sec:training_details}, on the AceReason‑1.1‑SFT~\citep{liu_acereason-nemotron_2025} and OpenThoughts3~\citep{guha2025openthoughtsdatarecipesreasoning} datasets, focused on mathematical reasoning and coding. A summary of all training hyperparameters is shown in \autoref{tab:hyperparameters}.
In total, training required 90 hours on a node with 8 H100 GPUs (80GB), corresponding to 720 GPU-hours. We use a sliding window of size $W=50$ in all experiments. Further details on the compute used for preliminary experiments, ablations, and testing are provided in Appendix~\ref{app:hyperparams}. We emphasize that the training cost of MELT from Ouro is just a tiny fraction of the total model training, as it requires only 32K samples. This is negligible relative to Ouro's overall training budget; for instance, it uses 16.4M tokens in its final training stage alone. Even accounting for MELT slower training throughput, MELT training would still require less than 5\% of the time spent on Ouro's final training stage.

To evaluate the reasoning capabilities of MELT, we benchmark the model on six mathematical reasoning benchmarks (AIME24~\citep{aime2024}, AIME25~\citep{aime2025},
AIME26~\citep{aime2026}, AMC23~\citep{amc2023},
MATH500~\citep{lightman2023lets},
OlympiadBench~\citep{he2024olympiadbench}) and four general reasoning benchmarks (GPQA~\citep{rein2023gpqagraduatelevelgoogleproofqa}, HLE~\citep{hle}, MMLU-Redux~\citep{gema2024mmlu, hendrycks_measuring_2021}, HumanEval+~\citep{chen2021evaluating}). For context, we compare MELT-1.6B performance with the state-of-the-art non-looped models of its size (Qwen3-1.7B~\citep{yang2025qwen3technicalreport}, Gemma4-E2B~\citep{gemma}, Qwen3.5-2B~\citep{qwen35blog}, DeepSeek-R1-1.5B~\citep{Guo_2025}), as well as the looped model Ouro‑1.4B‑Thinking~\citep{zhu2025scalinglatentreasoninglooped}, from which MELT‑1.6B is derived. We evaluate all models with LightEval v0.8.1, using the default benchmark prompts, extraction procedures, and evaluation settings, as detailed in Appendix~\ref{app:hyperparams}. Following~\citep{zhu2025scalinglatentreasoninglooped}, we use temperature $1.0$ and top-$p$ $0.7$; all evaluations use a maximum completion length of 32k tokens.

\subsection{Reasoning performance}
\label{sec:results}

\autoref{tab:perf_comparison_transposed} shows that MELT-1.6B consistently outperforms all non-looped baselines across both mathematical and general reasoning benchmarks, while maintaining a comparable memory footprint. In particular, in mathematical benchmarks MELT is only surpassed by Qwen3-1.7B and Gemma4-E2B on AMC23, and in other reasoning bencharks only by Qwen3.5-2B on GPQA. On all the other benchmarks, MELT achieves a superior performance, showing on average significant reasoning improvements, across both mathematical and general reasoning tasks.

MELT's performance is on par with Ouro's, despite retaining a full per-loop KV cache and therefore incurring substantially higher memory usage. Notably, MELT outperforms Ouro on AIME24 pass@1, AIME25 pass@10, and AIME26 on both pass@1 and pass@10. It also achieves better results across all non-math benchmarks, with an average improvement of 3\%. Overall, these results highlight MELT's strong performance–efficiency trade-off, matching the performance of memory-intensive looped architectures while requiring only a fraction of their memory footprint.

MELT's performance is on-par with Ouro's, even though the latter retains a full per-loop KV cache and thus has substantially higher memory usage. Notably, MELT outperforms Ouro on AIME24 pass@1, AIME25 pass@10, and AIME26 on both pass@1 and pass@10. It also achieves better results across all non-math benchmarks, with an average improvement of 3\%. Overall, these results highlight MELT's strong performance–efficiency trade-off, matching the performance of other looped architectures while requiring only a fraction of their memory footprint.

\begin{table}[t]
  \centering
  \caption{Performance comparison of MELT-1.6B and similar sized SOTA models across benchmarks. We use \textbf{bold} and \underline{underlining} to denote the best and second-best performance, respectively.}
\vspace{0.9\baselineskip}
  \label{tab:perf_comparison_transposed}

  \scriptsize
  \setlength{\tabcolsep}{2.5pt}
  \renewcommand{\arraystretch}{1.1}
  \resizebox{\linewidth}{!}{%
  \begin{tabular}{l l c G c c c c}
    \toprule
    \textbf{Dataset} & \textbf{Metric}
& \makecell{\textbf{Ouro-1.4B}\\\textbf{Thinking}}
& \makecell{\textbf{MELT}\\\textbf{1.6B}}
& \makecell{\textbf{Qwen3}\\\textbf{1.7B}}
& \makecell{\textbf{Gemma4}\\\textbf{E2B}}
& \makecell{\textbf{Qwen3.5}\\\textbf{2B}}
& \makecell{\textbf{DeepSeek-R1}\\\textbf{1.5B}}
\\
    \midrule

    \multirow{2}{*}{AIME24}
    & pass@1 & \underline{50.2}{\tiny\textcolor{gray}{$\pm$1.6}} & \textbf{50.4}{\tiny\textcolor{gray}{$\pm$1.6}} & 43.1{\tiny\textcolor{gray}{$\pm$1.5}} & 40.6{\tiny\textcolor{gray}{$\pm$1.8}} & 19.0{\tiny\textcolor{gray}{$\pm$1.3}} & 32.1{\tiny\textcolor{gray}{$\pm$1.6}} \\
    & pass@10
    & \textbf{81.5}{\tiny\textcolor{gray}{$\pm$1.9}} & \underline{77.5}{\tiny\textcolor{gray}{$\pm$2.4}} & 76.2{\tiny\textcolor{gray}{$\pm$2.6}} & 68.5{\tiny\textcolor{gray}{$\pm$2.8}} & 47.0{\tiny\textcolor{gray}{$\pm$2.5}} & 75.7{\tiny\textcolor{gray}{$\pm$3.5}} \\[0.1cm]

    \multirow{2}{*}{AIME25}
    & pass@1
    & \textbf{36.7}{\tiny\textcolor{gray}{$\pm$1.5}} & \underline{35.8}{\tiny\textcolor{gray}{$\pm$1.3}} & 33.1{\tiny\textcolor{gray}{$\pm$1.3}} & 26.5{\tiny\textcolor{gray}{$\pm$1.3}} & 16.9{\tiny\textcolor{gray}{$\pm$1.1}} & 20.4{\tiny\textcolor{gray}{$\pm$1.2}} \\
    & pass@10
    & \underline{69.0}{\tiny\textcolor{gray}{$\pm$2.5}} & \textbf{76.0}{\tiny\textcolor{gray}{$\pm$2.7}} & 58.6{\tiny\textcolor{gray}{$\pm$2.7}} & 50.1{\tiny\textcolor{gray}{$\pm$2.9}} & 37.1{\tiny\textcolor{gray}{$\pm$2.8}} & 46.0{\tiny\textcolor{gray}{$\pm$2.7}} \\[0.1cm]

    \multirow{2}{*}{AIME26}
    & pass@1
    & \underline{44.0}{\tiny\textcolor{gray}{$\pm$1.5}} & \textbf{45.0}{\tiny\textcolor{gray}{$\pm$1.6}} & 31.7{\tiny\textcolor{gray}{$\pm$1.4}} & 36.0{\tiny\textcolor{gray}{$\pm$1.7}} & 16.0{\tiny\textcolor{gray}{$\pm$1.3}} & 19.8{\tiny\textcolor{gray}{$\pm$1.3}} \\
    & pass@10
    & \underline{73.2}{\tiny\textcolor{gray}{$\pm$2.4}} & \textbf{78.9}{\tiny\textcolor{gray}{$\pm$2.0}} & 61.5{\tiny\textcolor{gray}{$\pm$2.9}} & 58.3{\tiny\textcolor{gray}{$\pm$2.6}} & 46.7{\tiny\textcolor{gray}{$\pm$2.4}} & 48.7{\tiny\textcolor{gray}{$\pm$2.7}} \\

    \multirow{2}{*}{AMC23}
    & pass@1
    & \underline{81.2}{\tiny\textcolor{gray}{$\pm$1.2}} & 77.2{\tiny\textcolor{gray}{$\pm$1.2}} & 79.2{\tiny\textcolor{gray}{$\pm$1.2}} & \textbf{82.7}{\tiny\textcolor{gray}{$\pm$1.1}} & 64.4{\tiny\textcolor{gray}{$\pm$1.4}} & 70.9{\tiny\textcolor{gray}{$\pm$1.3}} \\
    & pass@10
    & \underline{96.6}{\tiny\textcolor{gray}{$\pm$1.0}} & 96.5{\tiny\textcolor{gray}{$\pm$1.5}} & \textbf{97.8}{\tiny\textcolor{gray}{$\pm$1.5}} & 95.0{\tiny\textcolor{gray}{$\pm$1.9}} & 92.0{\tiny\textcolor{gray}{$\pm$1.2}} & 96.1{\tiny\textcolor{gray}{$\pm$1.4}} \\[0.1cm]

    MATH-500
    & accuracy
    & \textbf{94.4}{\tiny\textcolor{gray}{$\pm$1.0}} & \underline{93.4}{\tiny\textcolor{gray}{$\pm$1.1}} & 90.6{\tiny\textcolor{gray}{$\pm$1.3}} & 87.6{\tiny\textcolor{gray}{$\pm$1.5}} & 79.4{\tiny\textcolor{gray}{$\pm$1.8}} & 84.2{\tiny\textcolor{gray}{$\pm$1.6}} \\[0.1cm]

    OlympiadB
    & accuracy
    & \textbf{67.5}{\tiny\textcolor{gray}{$\pm$1.9}} & \underline{67.3}{\tiny\textcolor{gray}{$\pm$2.0}} & 63.5{\tiny\textcolor{gray}{$\pm$2.0}} & 62.7{\tiny\textcolor{gray}{$\pm$2.0}} & 48.4{\tiny\textcolor{gray}{$\pm$2.1}} & 54.2{\tiny\textcolor{gray}{$\pm$2.1}} \\[0.15cm]

    \rowcolor{gray!15}\textbf{Avg math} & pass@1 & \textbf{62.3} & \underline{61.5} & 56.9 & 56.0 & 40.7 & 46.9 \\
    \midrule

    GPQA& accuracy
    & 40.8{\tiny\textcolor{gray}{$\pm$2.3}} & \underline{42.9}{\tiny\textcolor{gray}{$\pm$2.3}} & 37.3{\tiny\textcolor{gray}{$\pm$2.3}} & 39.1{\tiny\textcolor{gray}{$\pm$2.3}} & \textbf{45.1}{\tiny\textcolor{gray}{$\pm$2.4}} & 31.9{\tiny\textcolor{gray}{$\pm$2.2}} \\[0.1cm]

    HLE& accuracy
    & \textbf{2.7}{\tiny\textcolor{gray}{$\pm$0.9}} & 1.3{\tiny\textcolor{gray}{$\pm$0.8}}
    & 1.3{\tiny\textcolor{gray}{$\pm$0.7}} & \underline{2.0}{\tiny\textcolor{gray}{$\pm$0.8}} & 1.7{\tiny\textcolor{gray}{$\pm$0.7}} & \underline{2.0}{\tiny\textcolor{gray}{$\pm$0.8}} \\[0.1cm]

    MMLU-Redux& accuracy
    & 74.2{\tiny\textcolor{gray}{$\pm$0.6}} & \textbf{77.2}{\tiny\textcolor{gray}{$\pm$0.6}}
    & 73.8{\tiny\textcolor{gray}{$\pm$0.6}} & \underline{75.3}{\tiny\textcolor{gray}{$\pm$0.6}} & \underline{75.3}{\tiny\textcolor{gray}{$\pm$0.6}} & 53.3{\tiny\textcolor{gray}{$\pm$0.7}} \\[0.1cm]

    HumanEval+& accuracy
    & \underline{76.8}{\tiny\textcolor{gray}{$\pm$3.3}} & \textbf{84.8}{\tiny\textcolor{gray}{$\pm$3.0}} & 71.3{\tiny\textcolor{gray}{$\pm$3.5}} & 61.6{\tiny\textcolor{gray}{$\pm$3.6}} & 26.2{\tiny\textcolor{gray}{$\pm$3.4}} & 57.3{\tiny\textcolor{gray}{$\pm$3.9}} \\[0.15cm]

    \rowcolor{gray!15}\textbf{Avg non-math} & pass@1 & \underline{48.6} & \textbf{51.6} & 45.9 & 45.5 & 37.1 & 36.1 \\

    \bottomrule
  \end{tabular}
  }
  \vspace{-14pt}
\end{table}

\subsection{Scaling to a bigger model}

We apply the same methodology to Ouro-2.6B-Thinking, without any model-specific modifications, to train MELT-3.1B. As shown in Table~\ref{tab:scaling_26b}, MELT-3.1B achieves performance comparable to Ouro-2.6B across all benchmarks and outperforms it on average across non-math benchmarks, demonstrating that the MELT approach effectively scales to a larger looped models.

% \begin{table}[t]
%   \centering
%   \caption{Performance of MELT-3.1B compared to Ouro-2.6B-Thinking.}
%   \label{tab:scaling_26b}
% \vspace{0.9\baselineskip}
%   \scriptsize
%   \resizebox{\linewidth}{!}{%
%   \begin{tabular}{lcccccccccccc}
%     \toprule
%     & \multicolumn{2}{c}{\textbf{AIME24}} & \multicolumn{2}{c}{\textbf{AIME25}} & \multicolumn{2}{c}{\textbf{AMC23}} & \textbf{MATH-500} & \textbf{Avg math} \\
%     \textbf{Model} & pass@1 & pass@10 & pass@1 & pass@10 & pass@1 & pass@10 & accuracy & pass@1 \\
%     \midrule
%     Ouro-2.6B-Thinking & 60.8 & 81.9 & 44.0 & 76.9 & 88.3 & 99.7 & 94.6 & 71.9 \\
%     MELT-3.1B & 56.7 & 86.7 & 42.7 & 63.3 & 88.3 & 97.5 & 93.8 & 70.4 \\
%     \midrule
    
%     & \multicolumn{2}{c}{\textbf{GPQA}} & \multicolumn{2}{c}{\textbf{MMLU-R}} & \multicolumn{2}{c}{\textbf{HumanEval+}} & \textbf{Avg non-math} \\
%     \midrule
%     Ouro-2.6B-Thinking & \multicolumn{2}{c}{50.0} & \multicolumn{2}{c}{84.8} & \multicolumn{2}{c}{79.3} & 71.4 \\
%     MELT-3.1B & \multicolumn{2}{c}{46.2} & \multicolumn{2}{c}{83.4} & \multicolumn{2}{c}{86.0} & 71.9 \\
%     \bottomrule
%   \end{tabular}%
%   }
%   \vspace{-12pt}
% \end{table}

\begin{table}[t]
  \centering
  \caption{Performance of MELT-3.1B compared to Ouro-2.6B-Thinking.}
  \label{tab:scaling_26b}
  \vspace{0.9\baselineskip}
  \scriptsize
  \resizebox{\linewidth}{!}{%
  \begin{tabular}{l cc cc cc c c}
    \toprule

    &
    \multicolumn{2}{c}{\textbf{AIME24}} &
    \multicolumn{2}{c}{\textbf{AIME25}} &
    \multicolumn{2}{c}{\textbf{AMC23}} &
    \textbf{MATH-500} &
    \textbf{Avg math} \\
    \cmidrule(lr){2-3}
    \cmidrule(lr){4-5}
    \cmidrule(lr){6-7}
    \cmidrule(lr){8-8}
    \cmidrule(lr){9-9}

    \textbf{Model}
    & pass@1 & pass@10
    & pass@1 & pass@10
    & pass@1 & pass@10
    & accuracy
    & pass@1 \\
    \midrule

    Ouro-2.6B-Thinking
    & \textbf{60.8} & 81.9
    & \textbf{44.0} & \textbf{76.9}
    & \textbf{88.3} & \textbf{99.7}
    & \textbf{94.6}
    & \textbf{71.9} \\

    MELT-3.1B
    & 56.7 & \textbf{86.7}
    & 42.7 & 63.3
    & \textbf{88.3} & 97.5
    & 93.8
    & 70.4 \\

    \midrule

    &
    \multicolumn{2}{c}{\textbf{GPQA}} &
    \multicolumn{2}{c}{\textbf{MMLU-R}} &
    \multicolumn{2}{c}{\textbf{HumanEval+}} &
    \multicolumn{2}{c}{\textbf{Avg non-math}} \\
    \cmidrule(lr){2-3}
    \cmidrule(lr){4-5}
    \cmidrule(lr){6-7}
    \cmidrule(lr){8-9}

    Ouro-2.6B-Thinking
    & \multicolumn{2}{c}{\textbf{50.0}}
    & \multicolumn{2}{c}{\textbf{84.8}}
    & \multicolumn{2}{c}{79.3}
    & \multicolumn{2}{c}{71.4} \\

    MELT-3.1B
    & \multicolumn{2}{c}{46.2}
    & \multicolumn{2}{c}{83.4}
    & \multicolumn{2}{c}{\textbf{86.0}}
    & \multicolumn{2}{c}{\textbf{71.9}} \\

    \bottomrule
  \end{tabular}%
  }
  \vspace{-12pt}
\end{table}

\subsection{Memory-Efficient Reasoning}
\label{subsec:memory_numbers}

In this subsection, we report exact inference-time memory usage measured with Hugging Face Transformers to evaluate how MELT compares with existing baselines in terms of memory efficiency. This analysis highlights the substantial improvements achieved by MELT compared to Ouro, since for long-context generation the dominant contributor to memory usage is the KV-cache. For each model, we report the following memory measurements, obtained on a single H100 GPU using a 200-token prompt and a 32k-token generation:

\begin{itemize}
    \item \textbf{Peak memory}: Maximum GPU memory usage observed during inference.
    \item \textbf{Model memory}: Memory required to store the model weights.
    \item \textbf{KV-cache}: Total KV-cache memory consumed at the end of the generation.
    \item \textbf{Average per-token KV-cache}: Average KV-cache memory consumed per generated token.
    \item \textbf{Max batch}: Maximum batch size that fits on a single H100 GPU.
\end{itemize}

As shown in Table~\ref{tab:memory_usage}, Ouro exhibits the largest KV-cache footprint, as its loop-specific KV growth causes memory to scale linearly with the number of reasoning loops. In contrast, MELT decouples reasoning depth from KV growth by maintaining a constant-size latent state instead of appending new KV entries, reducing memory by $\sim\!3\times$. Although Qwen remains slightly more memory efficient in KV usage, the gap is small: for a 32k-token generation, Ouro exceeds Qwen by $\sim20$ GB, while MELT is only $\sim2.5$ GB more expensive. This difference stems from Qwen’s use of \ac{GQA}, which reduces KV memory by sharing keys and values across query heads, whereas MELT does not employ \ac{GQA}. Additional inference metrics, including TTFT, prefill throughput, decode throughput, total throughput, and prefill peak memory, are reported in Appendix~\ref{app:efficiency}.

\begin{table}[t]
    \centering
    \caption{
    Exact KV-cache memory measured with Hugging Face Transformers and derived VRAM requirements for generating a 32k-token sequence. MELT's KV-cache per token refers to the average of all tokens.
    }
    \label{tab:memory_usage}
\vspace{0.9\baselineskip}
    \small
    \setlength{\tabcolsep}{7pt}
    \begin{tabular}{lccccc}
\toprule
        &
        \textbf{Peak mem.} &
        \textbf{Model mem.} &
        \textbf{KV-cache} &
        \textbf{Avg. per-token} &
        \textbf{Max batch} \\
        \textbf{Model} &
        \textbf{(GB)} &
        \textbf{(GB)} &
        \textbf{(GB)} &
        \textbf{KV-cache (MB)} &
         \\
        \midrule
        Ouro-1.4B-Thinking  & 26.58 & 2.87 & 23.71 & 0.759 & 11 \\
        MELT-1.6B & 9.23 & 3.27 & 5.96 & 0.191 & 3 \\
{\footnotesize \hspace{8pt} Memory improvement} & {\footnotesize \hspace{8pt} $\times$2.88} & {\footnotesize \hspace{8pt} --} & {\footnotesize \hspace{6pt} $\times$3.98} & {\footnotesize \hspace{4pt} $\times$3.98} & { \footnotesize \hspace{0pt} --}\\
        \midrule
        Qwen3-1.7B     & 6.74  & 3.44 & 3.30 & 0.106 & 21 \\
        \bottomrule
    \end{tabular}
  \vspace{-13pt}
\end{table}

We next investigate how these memory savings translate into performance under fixed memory budgets. To study this setting, we cap the maximum sequence length available to each model and evaluate the resulting memory--accuracy trade-off. This reflects a practical deployment constraint: under a fixed memory budget, a model with a smaller KV cache can spend more tokens on solving a problem, while a model with a larger KV cache may be forced to stop earlier. As shown in \autoref{fig:first_image}(a), MELT occupies most of the Pareto frontier between total memory and accuracy. In the low- and mid-memory regimes, MELT achieves higher accuracy than similarly sized non-looped baselines, and even surpasses Qwen3-1.7B under comparable memory budgets. Thus, MELT does not merely reduce memory consumption in isolation, but converts these memory savings into a stronger accuracy--memory trade-off across realistic sequence-length budgets.

\subsection{Ablation studies}
\label{subsec:ablations}

\subsubsection{Gate mechanism variants}
A core component of MELT is its gated update mechanism, which controls how loop-specific information is accumulated into the latent state. To assess the necessity and effectiveness of this design, we train a set of variants in which the proposed element-wise gating mechanism is replaced with simpler aggregation schemes. All other components are kept identical. Concretely, we compare the full MELT model against the following variants:

% \emph{Mean}: the KV cache is computed as the average of the KV representations produced by all loops up to the current step.
% \emph{EMA-0.2}: the KV cache is computed as an exponential moving average (EMA) of the KV representations up to the current step. The chosen decay factor (0.2) matches the average gate value observed in our trained MELT models. This is equivalent to the gated mechanism with gate value fixed to 0.2.
% \emph{Last}: the KV cache is constructed solely from the final reasoning loop, discarding information from earlier loops.
% \emph{Single-gated}: the gated update is replaced with a scalar gate per token, such that a single gating value modulates the entire hidden state uniformly, rather than using an element-wise gate.

\begin{itemize}
    \item \textbf{Mean}: the KV cache is computed as the average of the KV representations produced by all loops up to the current step.
    \item \textbf{Ema-0.2}: the KV cache is computed as an exponential moving average (EMA) of the KV representations up to the current step. The chosen decay factor (0.2) matches the average gate value observed in our trained MELT models.
    \item \textbf{Last}: the KV cache is constructed solely from the final reasoning loop, discarding information from earlier loops.
    \item \textbf{Single-gated}: the gated update is replaced with a scalar gate per token, such that a single gating value modulates the entire hidden state uniformly.
\end{itemize}

Table~\ref{tab:gate_ablation} shows that MELT (element-wise gating) achieves the best performance, being at least 1\% better on average than the other variants, highlighting the importance of gating aggregation. Among variants without additional parameters, \textit{Last} performs the best, with a big difference compared to \textit{Mean} and slightly surpassing \textit{EMA-0.2} and even \textit{Single-gated}. Having into account that \textit{Last} average is close to MELT's and that the mechanism is simpler and uses less parameters, it is a good alternative to take into consideration.

\begin{table}[t]
  \caption{MELT's gating mechanism ablation. \textbf{Bold} denotes best performance. \hfill \vphantom{.}}
  \label{tab:gate_ablation}
  \vspace{0.9\baselineskip}
  \centering
  \small

  \resizebox{\linewidth}{!}{
  \begin{tabular}{l cc cc cc c c}
  \toprule
  &
      \multicolumn{2}{c}{\textbf{AIME24}} &
      \multicolumn{2}{c}{\textbf{AIME25}} &
      \multicolumn{2}{c}{\textbf{AMC23}} &
      \multirow{2}{*}{\textbf{MATH-500}} &
      \multirow{2}{*}{\textbf{Avg}} \\
      \cmidrule(lr){2-3}\cmidrule(lr){4-5}\cmidrule(lr){6-7}
  \textbf{Model Variant}  & pass@1 & pass@10
      & pass@1 & pass@10
      & pass@1 & pass@10
      & & \\
  \midrule
  Mean          & 38.3 & 66.3          & 25.6 & 51.0          & 68.3          & 95.6          & 90.2          & 55.6 \\
  EMA-0.2       & 44.4 & \textbf{78.7} & 33.1 & 70.0          & 79.5          & \textbf{97.1} & 91.6          & 62.2 \\
  Last          & 44.6 & 77.7          & 33.5 & 65.8          & \textbf{81.7} & 96.5          & 93.4          & 63.3 \\
  Single-gated  & 46.3 & 73.8          & 30.2 & 64.2          & 78.4          & 96.8          & \textbf{94.2} & 62.3 \\
  MELT-1.6B     & \textbf{50.4} & 77.5 & \textbf{35.8} & \textbf{76.0} & 77.2 & 96.5 & 93.4 & \textbf{64.2} \\
  \midrule
  \end{tabular}
  }

  \vspace{-10pt}
\end{table}

\subsubsection{Component removal}

We next perform a component removal ablation to assess the importance of individual training components in MELT. Starting from the full model, we progressively remove elements of the training procedure one by one, following their order of introduction, and fully retrain the model after each removal to ensure a fair comparison.
Specifically, we remove training mechanisms one by one to isolate their contribution, following the sequence: (i) removing knowledge distillation on all loops, reducing training to standard SFT; (ii) replacing chunk-wise training with fully parallel SFT; (iii) removing all training, using Ouro weights with randomly initialized gates.

\begin{table}[tb]
  \caption{Component removal ablation for MELT. Components are progressively removed one by one (top to bottom). \textbf{Bold} denotes best performance.}
  \label{tab:component_ablation}
  \centering
  \vspace{0.9\baselineskip}
  \small

  \resizebox{\linewidth}{!}{
  \begin{tabular}{l cc cc cc c}
  \toprule
  &
      \multicolumn{2}{c}{\textbf{AIME24}} &
      \multicolumn{2}{c}{\textbf{AIME25}} &
      \multicolumn{2}{c}{\textbf{AMC23}} &
      \textbf{MATH-500} \\
      \cmidrule(lr){2-3}\cmidrule(lr){4-5}\cmidrule(lr){6-7}\cmidrule(lr){8-8}
  \textbf{Component removed} & pass@1 & pass@10
      & pass@1 & pass@10
      & pass@1 & pass@10
      & accuracy \\
  \midrule

  MELT-1.6B               & \textbf{50.4} & \textbf{77.5} & \textbf{35.8} & \textbf{76.0} & \textbf{77.2} & \textbf{96.5} & \textbf{93.4} \\
  $-$ GKD-allLoops        & 45.8          & 72.0          & 33.1          & 67.1          & 75.8          & 95.6          & 91.4 \\
  $-$ Chunk-wise training & 10.6          & 37.7          & 11.3          & 28.2          & 38.1          & 72.5          & 55.0 \\
  $-$ Training            & 14.2          & 42.3          & 12.3          & 32.4          & 44.5          & 79.5          & 58.0 \\

  \midrule
  \end{tabular}
  }

  \vspace{-15pt}
\end{table}

As shown by Table~\ref{tab:component_ablation}, the two main components yield a clear and consistent improvement over the preceding configuration across all benchmarks. Disabling knowledge distillation on all loops already causes a notable performance drop in most benchmarks, demonstrating its critical role in aligned the learned MELT representations. Further removing chunk-wise training leads to almost complete failure, halving the models capabilities. Interestingly, fully-parallel training degrades performance compared to using the original Ouro weights with random gate initialization, confirming that respecting MELT’s sequential KV dynamics during training is indispensable. Overall, these results demonstrate that MELT’s performance arises from the cumulative effect of these two capital components. 

% We also provide a training-data ablation in Appendix~\ref{app:data_ablations}, where we fine-tune Qwen3-1.7B and Ouro-Thinking-1.4B on the same data used for MELT to isolate architectural gains from data effects.

\section{Limitations and future work} \label{sec:limitations}

A limitation inherited from Ouro is that the number of recurrent loops is fixed at inference time. While this provides a simple mechanism to control compute, it does not account for the fact that different inputs and tokens may require varying amounts of reasoning. Notably, \ac{MELT}'s memory remains independent of loop depth except for the small last-$W$ sliding window, making it particularly well-suited for future extensions with adaptive loop depth, enabling dynamic allocation of reasoning steps based on input complexity.
Another limitation carried over from Ouro is that our current implementation does not yet explore \ac{GQA}. Extending \ac{MELT} to grouped query attention remains an important direction, as \ac{GQA} can reduce memory bandwidth and KV cache overhead during inference. Thanks to \ac{MELT}'s constant-memory design, combining it with \ac{GQA} is especially promising and could further improve efficiency, potentially closing the remaining gap in memory usage with standard Transformer baselines.
The current experiments adapt pretrained Ouro checkpoints; incorporating the MELT mechanism throughout pretraining remains future work and could enable models that are not constrained by the inherited design choices discussed above.
Finally, \ac{MELT} requires sequential KV updates during training, which constrains parallelism compared to standard transformer training. While our chunk-wise training and distillation procedure provides a practical adaptation path, developing more parallelizable training strategies remains an important direction for scaling \ac{MELT} to larger models, longer reasoning horizons, and broader application domains.

\section{Conclusion} \label{sec:conclusions}

We introduce \ac{MELT}, an architecture that enables deep latent reasoning in looped transformers by decoupling memory usage from reasoning depth. By replacing the append‑only KV‑cache with a \textbf{gated, constant‑size latent state}, \ac{MELT} allows inference‑time compute to scale without incurring linear memory growth. When integrated into Ouro, it achieves competitive reasoning performance relative to similarly sized baselines.

Training \ac{MELT}, however, requires particular care. Its KV cache computation introduces a sequential dependency across tokens, preventing the fully parallel token processing typically used during training. We address this challenge through \textbf{chunk-wise training}, which provides a controllable trade-off between inference fidelity (smaller chunks) and training efficiency (larger chunks). In addition, we apply knowledge distillation with supervision on all loops, to keep MELT aligned with the original model despite the substantial architectural changes introduced by MELT. Importantly, given the minimal cost of adapting a looped transformer to MELT and the additional training overhead introduced by chunk-wise training, we propose MELT as a lightweight post-training adaptation of an existing looped model. 
% The cost of converting Ouro into MELT is only a tiny fraction of the original training budget, requiring just 32K training samples. For comparison, Ouro consumes 16.4M tokens in its final training stage alone. Consequently, even accounting for MELT's lower training throughput due to chunk-wise training, the adaptation would require less than 5\% of the time spent on Ouro's final training stage.
Converting Ouro to MELT requires only 32K training samples, versus the 16.4M tokens used in Ouro's final training stage. Even with MELT's lower chunk-wise training throughput, the adaptation takes less than 5\% of the time of Ouro's final training stage.

Empirically, MELT delivers strong and consistent performance across both mathematical and general reasoning benchmarks. Notably, it surpasses similarly sized standard Transformer baselines while operating under the same constant memory budget, demonstrating that improved reasoning capability can be achieved without increasing memory. To our knowledge, MELT is the first architecture to exceed the performance of standard models with the same memory footprint. These results highlight the effectiveness of looped architectures and demonstrating that iterative computation provides meaningful gains in reasoning even when memory is strictly constrained.

\bibliography{references}
\bibliographystyle{unsrtnat}

% APPENDIX
\newpage
\appendix
\onecolumn

\section{Extended Related Work} \label{app:related}

\paragraph{Looped transformers.}
While \ac{CoT} \citep{wei2022chain} and other ITC techniques have recently been highly influential, a complementary direction has emerged that focuses on vertical reasoning via recurrent architectures.
Simple recurrent architectures such as HRM \citep{wang_hierarchical_2025} and TRM \citep{jolicoeur2025less} have demonstrated strong performance on targeted reasoning tasks, while transformer-based models have also been modified to incorporate looping mechanisms.
For instance, \citet{li2025skip_or_loop_test_time_depth} propose a per-sample test-time depth adaptation, where a pretrained LLM’s layers are treated as modules that can be skipped or repeated (looped) and reordered to form a sample-specific chain-of-layers.
On a similar direction, \citet{fu2025think} introduced adaptive computation budgets, employing a classifier to dynamically allocate additional latent iterations for difficult tokens that can be skipped or repeated (looped) and reordered to form a sample-specific chain-of-layers.

Most notably within this line of work, looped transformers have emerged as a powerful architectural choice.
Existing studies on simplified setups have shown that, compared to similar-sized vanilla transformers, looped transformers show superior capacity at multi-hop reasoning \citep{saunshi2025latent_thoughts_looped_transformers, kohli2026loopthinkgeneralize}, lenght generalization \citep{fan2025looped_length_generalization, kohli2026loopthinkgeneralize} and learning algorithms \citep{yang2024looped_learning_algorithms}. 
Despite the challenges of scaling such architectures when unrolled over many steps, most notably optimization instability and vanishing gradients \citep{dehghani2019universal}, recent studies demonstrate that looped transformers can already be trained stably at the scale of several billion parameters \citep{zhu2025scalinglatentreasoninglooped, geiping2025recurrent_depth_latent_reasoning}, including reasoning‑focused models \citep{zhu2025scalinglatentreasoninglooped}.
Existing approaches generally follow two strategies: either looping transformer layers directly across iterations \citep{zhu2025scalinglatentreasoninglooped}, or looping only a central subset of layers using a middle‑cycle design with fixed prelude and coda blocks \citep{geiping2025recurrent_depth_latent_reasoning, zeitoun2026hyperlooptransformers}. 
Additional work has reported favorable scaling behavior \citep{prairie2026parcaescalinglawsstable}, proposed architectural refinements \citep{zeitoun2026hyperlooptransformers}, and provided mechanistic insights into the internal dynamics of looped models \citep{blayney2026mechanisticanalysisloopedreasoning}, further supporting looped transformers as a robust and promising research direction.

\paragraph{KV cache compression and vertical sharing.}
Efficient KV cache management is central in scaling recurrent models and long-context regimes, where cache cost scales with effective depth. Beyond head-level sharing via \ac{MQA} \citep{shazeer2019mqa} and \ac{GQA} \citep{ainslie2023gqa}, recent work exploits vertical redundancy across layers or recurrence steps. \ac{CLA} \citep{brandon2024cross} shows that KV representations remain relatively stable across adjacent layers, allowing multiple layers to read from a shared cache. \citet{deepseek2024} further compresses this information via \ac{MLA} using low-rank projections.

More specifically, there are some works that address the growing KV Cache issue in looped transformers. \citet{wu2025parallellooptransformerefficient} propose a KV-cache computation method that combines a global component from the first loop with a local component attending to a sliding window of recent tokens in the current loop.
\citet{bae2025mixtureofrecursionslearningdynamicrecursive} propose two key mechanisms for KV cache efficiency: Recursion-wise Caching, which selectively updates and attends only to the KV pairs of active "thinking" tokens at each depth, and Recursive KV Sharing, which reuses the initial cache from the first loop across all subsequent steps.
\citet{geiping2025recurrent_depth_latent_reasoning} show that the looped transformer, without additional training, naturally reuses the first loop KV‑cache entry from previous tokens, independent of when recurrence stops, and can further compress the cache by sharing entries periodically across recurrent steps.
Finally, \citet{zhu2025scalinglatentreasoninglooped} show that, without additional training, retaining only the last or average KV-cache and reusing it across all loops can preserve performance on some tasks.
Although some of these approaches achieve no or moderate performance drops in their respective settings, their applicability to practical long-reasoning remains unclear as their evaluations are mostly limited to weaker looped models and short or constrained generation settings. In fact, our analysis (Appendix~\ref{app:untrained_ablations}) suggests that directly applying untrained cache-sharing methods to Ouro, including those proposed by \citet{geiping2025recurrent_depth_latent_reasoning} and \citet{zhu2025scalinglatentreasoninglooped}, significantly degrades performance on long reasoning tasks.

\paragraph{Activation-level knowledge distillation.}
An extensive body of literature exists on \ac{KD}~\citep{hinton_distilling_2015}. We build on the line of work that aligns intermediate representations rather than relying solely on final output logits. Notable examples within this vast space include foundational works that distill internal states to compress structural knowledge~\citep{aguilar_knowledge_2020} and cross-layer mechanisms designed to calibrate semantic alignment between teacher and student~\citep{chen_cross-layer_2021}. In the context of large language models,~\citet{muralidharan_compact_2024} recently demonstrated the efficacy of combining structural pruning with layer-wise KD to derive highly accurate, compact models without full retraining. Furthermore, recent frameworks emphasize strict internal activation alignment to prevent catastrophic drift~\citep{hao_token_2025} and preserve complex, multi-step reasoning trajectories~\citep{fang_knowledge_2026}. Building on these foundations, our approach addresses representation drift in continuous recurrent architectures by applying knowledge distillation across all MELT loops.

\section{Analysis of existing KV-Cache sharing methods on long reasoning}
\label{app:untrained_ablations}

\begin{figure}[H]
\centering
\begingroup
\setlength{\fboxsep}{8pt}

\noindent\colorbox{gray!15}{%
  \parbox{\dimexpr\linewidth-2\fboxsep\relax}{%
    \scriptsize
    <think>Okay, so I need to simplify tan 100 degrees plus 4 times sin 100 degrees. Hmm, let me think about how to approach this. I remember some trigonometric identities, maybe I can use them here. \\

    First, let me write down the expression: tan(100°) + 4 sin(100°). Hmm, 100 degrees is in the second quadrant, so tan is negative there because tangent is sine over cosine, and cosine is negative in the second quadrant. Sin is positive. But maybe there's a way to combine these terms. \\

    Alternatively, maybe I can use some angle addition formulas or convert tangent into sine and cosine? Let me see. Since tangent is sine over cosine, so tan(100°) = sin(100°)/cos(100°). So then the expression becomes [sin(100°)/cos(100°)] +4 sin(100°). \\ \\

    Hmm, so if I get a common denominator, maybe I can combine them? Let's see. The common denominator would be cos(100°). So that would be [sin(100°) +4 sin(100°) cos(100°)] divided by cos(100°). Wait, but that might not help much. Hmm. \\

    Alternatively, maybe there is a clever way to write this as some sine or cosine function with some angle? Let me think. \\

    Wait, 100 degrees is 80 degrees plus 20 degrees. 80 is 60 plus 20, but not sure. Maybe I can use some addition formulas?\\

    Alternatively, perhaps there is a way to express 4 sin(100°) as some sine of another angle. Maybe with some multiple angle formulas?\\

    Alternatively, maybe I can write 100° as 60° +40°, because h, 60+40 is h 100. Let me see. Let me think.\\

    Wait, maybe I can use the identity for tan(A+B), but not sure. Alternatively, maybe I can write 100° as h 80°+20°, but maybe maybe I can't think of a direct identity.\\

    Alternatively, maybe I can use the identity that tan(x) +4 sin(x) = something. Maybe there is a formula for that?\\

    Wait, let me try tohe following:\\

    Let me consider the expressionation: tan(100° +4 sin(100°). Maybe I can writet it as [tan(100°) +4 sin(100°)]. to see if I can findd some angle such that this expressionation is equal to some sine or cosine functiontian ofthat angle.\\
    
    Hmm, maybe maybe maybe maybe maybe I know maybe that maybe the expressionation can be simplified using the identity:\\

    tan(A) +4 sin(A)=sqrt(5) something.\\

    Wait, I think I heard of a formula that tan(( +4 sin()=sqrt(5) or something like like like. Let me check that\\

    Let me try to see this. Let me think that if A=1 , then the expressionation would is tan()+=4 sin(), but probably that the value of thetan 100° plus h teh 4 sin100° is approximately equal to teh sqrt(5).\\ \\

    h, let me seech that. \\

    h,\\

    Hmmm, let me seeh theangleng in that.\\

    h, maybe,\\

    ,\\\\

    h,\\ \\

    Wait, \\
    h, maybe,\\
    h,
    h,

    the expressionason tan(h  +4 sin(h')= teh sqrt(5). \\

    Maybe, (...)\\
  }
}
\endgroup
\caption{Example reasoning trace in Ouro-1.4B-Thinking illustrating the failure mode of last-loop KV-cache sharing with prefill.}
\label{fig:degenerate_reasoning}

\end{figure}

As discussed in the related work, several studies on KV‑cache sharing in looped models suggest that caches can be reused to reduce memory footprint with minimal performance degradation. However, their applicability to practical long‑reasoning settings remains unclear, as existing evaluations are largely restricted to weaker looped models and short or otherwise constrained generation regimes. In this section, we investigate whether these methods extend to state‑of‑the‑art looped models (Ouro \citep{zhu2025scalinglatentreasoninglooped}) on long‑horizon reasoning tasks.

To do so, we evaluate several untrained KV-cache sharing strategies in which the cache is reused across recurrent loops. Specifically, we consider sharing the KV cache from the last loop, following~\citet{zhu2025scalinglatentreasoninglooped}, and from the first loop, following~\citet{geiping2025recurrent_depth_latent_reasoning}. Since~\citet{zhu2025scalinglatentreasoninglooped} observe that preserving the full prompt KV cache can be beneficial, we evaluate both variants for each strategy: one that keeps the original KV cache for the prompt and one that applies the first/last-loop sharing rule to the prompt as well.

All four untrained KV-sharing variants obtain zero performance on several reasoning benchmarks (full results can be seen in Table ~\ref{tab:untrained_kv_sharing_ablation}). This contrasts with~\citet{zhu2025scalinglatentreasoninglooped}, who report that last-loop KV sharing with prefill achieves performance comparable to Ouro on GSM8K and MATH-500. We hypothesize that this discrepancy is partly due to their few-shot CoT evaluation setup (Appendix~C.1 in \citet{zhu2025scalinglatentreasoninglooped}), which constrains the generation format and reduces the likelihood of the model drifting during decoding. In our setting, the failure mode is apparent in qualitative outputs: as illustrated in Figure~\ref{fig:degenerate_reasoning}, generations with last-loop sharing and prefill often begin coherently but eventually degenerate during extended reasoning. This behavior is consistent with an accumulation of KV-cache mismatch errors: as generation proceeds farther from the prompt, the shared cache increasingly deviates from the cache produced by the original model. More specifically, queries from earlier loops attend to keys and values produced at the final loop, creating a recurrent-depth mismatch with the original model's depth-matched states. Because standard KV values are used during prompt prefilling, an increasing fraction of the context is generated under the mismatched cache as more tokens are generated, further amplifying this effect. These results motivate the need for a constant-memory KV-cache method that remains stable over long reasoning traces, which is precisely the goal of our approach.

% see Table~\ref{tab:untrained_kv_sharing_ablation}, we observe that all four methods get zero performance in a bunch of reasoning tasks. 
% This contrasts, for eample, with \citet{zhu2025scalinglatentreasoninglooped}, where they find that last with prefill achieves comparable performance to ouro on gsm8k and math500. this is because they use few-shot prompting (see Appendix C.1 in \citet{zhu2025scalinglatentreasoninglooped}), which restrics the generation and helps the model not deviate from 
% if we analyse a response for last with prefill, see Figure~\ref{fig:degenerate_reasoning}, this makes total sense: the generation starts making sense, but at some point during the reasoning it degenerates. this is because the further you are from the prompt, the more different the kv cache is to the original one and the errors accumulate. this clearly justifies the need for a constant kv cache method that supports long reasoning traces, which is precisely the point of this paper. 

\begin{table}[H]
\caption{Performance of untrained KV-cache sharing strategies for Ouro-1.4B-Thinking. We compare first-loop and last-loop KV-cache sharing, with and without preserving the original prompt KV cache.}
\label{tab:untrained_kv_sharing_ablation}
\centering
\vspace{0.9\baselineskip}
\small
\hspace*{-8pt}
  \begin{tabular}{l cc cc cc c}
\toprule
&
    \multicolumn{2}{c}{\textbf{AIME24}} &
    \multicolumn{2}{c}{\textbf{AIME25}} &
    \multicolumn{2}{c}{\textbf{AMC23}} &
    \multirow{2}{*}{\textbf{MATH-500}} \\
    \cmidrule(lr){2-3}\cmidrule(lr){4-5}\cmidrule(lr){6-7}
 & pass@1 & pass@10
    & pass@1 & pass@10
    & pass@1 & pass@10
    &  \\
\midrule
Ouro-1.4B-Thinking & \textbf{50.2} & \textbf{81.5} & \textbf{36.7} & \textbf{69.0} & \textbf{81.2} & \textbf{96.6} & \textbf{94.4}  \\
Last loop (w/ prefill)  & 0.0 & 0.0 & 0.0 & 0.0 & 0.0 & 0.0 & 0.0 \\
Last loop  & 0.0 & 0.0 & 0.0 & 0.0 & 0.0 & 0.0 & 0.0 \\
First loop (w/ prefill)  & 0.0 & 0.0 & 0.0 & 0.0 & 0.0 & 0.0 & 0.0 \\
First loop  & 0.0 & 0.0 & 0.0 & 0.0 & 0.0 & 0.0 & 0.0 \\
\midrule
\end{tabular}
\end{table}

\section{Hyperparameters} \label{app:hyperparams}

This section provides the hyperparameters required to reproduce our training and evaluation runs. Table~\ref{tab:hyperparameters} reports the hyperparameters used for MELT-1.6B and the ablation studies, while Table~\ref{tab:hyperparameters_benchmarks} summarizes the benchmark evaluation settings. Note that for HLE, we evaluate on a 300-sample subset of the original dataset.

\begin{table}[!ht]
  \centering
\caption{Training hyperparameters for \textbf{MELT-1.6B}. MELT-3.1B uses the same training hyperparameters, with only architecture-specific quantities (e.g., hidden dimension, number of layers, parameter initialization) adjusted accordingly.}
\label{tab:hyperparameters}
  
  \begin{tabular}{l l}
\toprule
\textbf{Parameter} & \textbf{Value} \\
\midrule
Dataset-mix & 50\% AceReason-1.1-SFT, 50\% OpenThoughts3 \\
\# layers & 24 \\
Hidden dimension & 2048 \\
Recurrent steps & 4 \\
\# gating params. & $24 \times 2048^2 \times 2 \approx 0.2$B  \\
% Gating params. init. & $\sim \mathcal{N}_{[-0.04,\,0.04]}(0,\;0.02^2)$ \\
Original params. init. & Ouro-1.4B-Thinking \\
Chunk size (tokens) & 500 \\
Sliding-window size (tokens) & 50 \\
Batch size (tokens) & 320K \\
Seq. length (tokens) & 10K \\
Gradient norm clipping & $1.0$ \\
LR scheduler & Cosine decay with warmup \\
Optimizer & Adam ($\beta_1{=}0.9$, $\beta_2{=}0.95$) \\
Weight decay & $1.0\times10^{-4}$ \\
Learning rate & $8\times10^{-6}$ \\
Gate learning rate & $5.0\times10^{-4}$ \\
Training steps & 1000 \\
Training tokens & 320M \\
\bottomrule
\end{tabular}
\end{table}

\begin{table}[!ht]
  \centering
  \caption{Benchmark evaluation settings. Default templates are from LightEval v0.8.1}
  \label{tab:hyperparameters_benchmarks}
  \scriptsize
  \resizebox{\linewidth}{!}{%
  \begin{tabular}{llllrr}
    \toprule
    \textbf{Benchmark} & \textbf{Hugging Face dataset} & \textbf{Prompt template} & \textbf{Metric} & \textbf{Samples} & \textbf{Completions} \\
    \midrule
    AIME 2024 & \texttt{HuggingFaceH4/aime\_2024} & \texttt{MATH\_QUERY\_TEMPLATE} (default) & pass@1, pass@10 & 30 & 16 \\
    AIME 2025 & \texttt{yentinglin/aime\_2025} & \texttt{MATH\_QUERY\_TEMPLATE} (default) & pass@1, pass@10 & 30 & 16 \\
    AIME 2026 & \texttt{MathArena/aime\_2026} & \texttt{MATH\_QUERY\_TEMPLATE} (default) & pass@1, pass@10 & 30 & 16 \\
    AMC 2023 & \texttt{math-ai/amc23} & \texttt{MATH\_QUERY\_TEMPLATE} (default) & pass@1, pass@10 & 40 & 16 \\
    MATH-500 & LightEval built-in & \texttt{MATH\_QUERY\_TEMPLATE} (default) & accuracy & 500 & 1 \\
    OlympiadBench & \texttt{math-ai/olympiadbench} & \texttt{MATH\_QUERY\_TEMPLATE} (default) & accuracy & 581 & 1 \\
    GPQA Diamond & \texttt{Idavidrein/gpqa} & \texttt{GPQA\_QUERY\_TEMPLATE} (default) & accuracy & 448 & 1 \\
    HLE & \texttt{cais/hle} & \texttt{MATH\_QUERY\_TEMPLATE} (default) & accuracy & 300 & 1 \\
    MMLU-Redux & \texttt{edinburgh-dawg/mmlu-redux-2.0} & \texttt{GPQA\_QUERY\_TEMPLATE} (default) & accuracy & 5700 & 1 \\
    HumanEval+ & \texttt{evalplus/humanevalplus} & \texttt{HUMANEVAL\_QUERY\_TEMPLATE} (custom) & accuracy & 164 & 1 \\
    \bottomrule
  \end{tabular}%
  }
\end{table}

\begin{table}[!ht]
  \centering
  \caption{vLLM parameters used for benchmark evaluation}
  \label{tab:vllm_parameters}
  \begin{tabular}{ll}
    \toprule
    \textbf{Parameter} & \textbf{Value} \\
    \midrule
    temperature & 1.0 \\
    top-$p$ & 0.7 \\
    max new tokens & 32768 \\
    dtype & bfloat16 \\
    tensor parallel size & 1 \\
    pipeline parallel size & 1 \\
    data parallel size & 1 \\
    GPU memory utilization & 0.9 \\
    \bottomrule
  \end{tabular}
\end{table}

Regarding compute, the main training run for MELT-1.6B required 90 hours on a node with 8 H100 GPUs (80GB), corresponding to 720 GPU-hours. The three ablation runs also required 60 hours each on the same 8-GPU node, for a total of 1,440 GPU-hours. Evaluation required approximately 500 GPU-hours, while preliminary experiments accounted for roughly 15,000 GPU-hours. Overall, the project used approximately 17,000 GPU-hours.

\begin{figure}[t]
\centering
    \centering
    \includegraphics[width=0.7\linewidth]{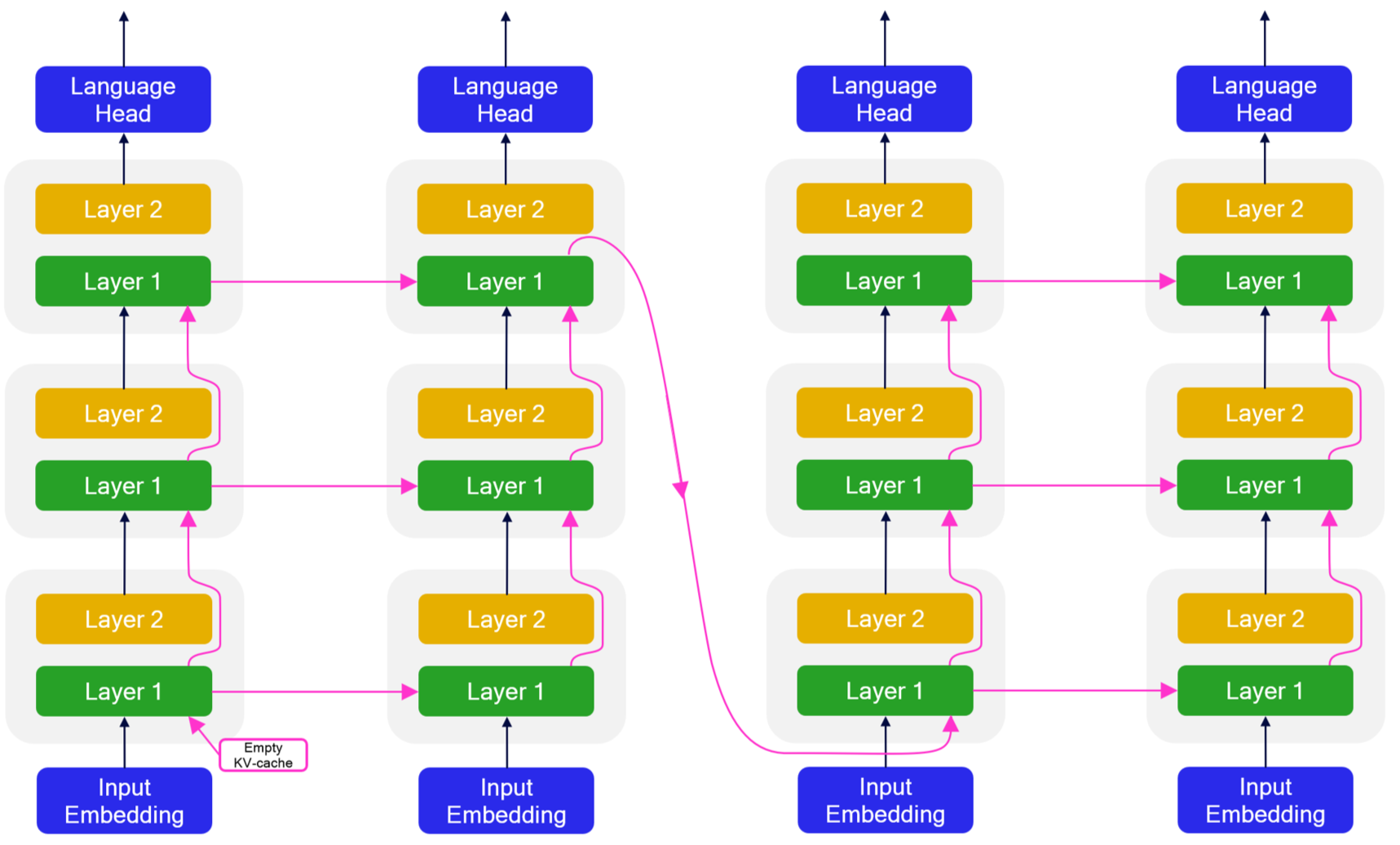}
  \caption{
    Visualization of \textit{chunk-wise trianing}. Example with sequence length 4 and chunk size 2. MELT’s KV cache is computed in parallel within each chunk and sequentially across chunks, balancing training efficiency with a closer approximation to autoregressive inference.
  }
  \label{fig:method_training}
  \vspace{-12pt}
\end{figure}
\section{Efficiency analysis}
\label{app:efficiency}

We report additional efficiency measurements using Hugging Face Transformers.

\subsection{Inference efficiency}

Table~\ref{tab:inference_efficiency} uses one H100 GPU, a 200-token prompt, and a 32k-token generation. MELT reduces decode memory and supports a larger batch than Ouro, at the cost of slower batch-1 decoding. MELT preserves parallel prefill by constructing each loop's cache from previous-token hidden states at the same loop.

\begin{table}[!ht]
  \centering
  \caption{Inference efficiency metrics for MELT-1.6B, Ouro-1.4B-Thinking, and Qwen3-1.7B.}
  \label{tab:inference_efficiency}
  \small
  \resizebox{\linewidth}{!}{%
  \begin{tabular}{lrrrrrrr}
    \toprule
    \textbf{Model} & \textbf{TTFT} & \textbf{Prefill} & \textbf{Decode} & \textbf{Total} & \textbf{Prefill peak} & \textbf{Decode peak} & \textbf{Max batch} \\
    & \textbf{(s)} & \textbf{(tok/s)} & \textbf{(tok/s)} & \textbf{(tok/s)} & \textbf{(GB)} & \textbf{(GB)} & \textbf{(@32k)} \\
    \midrule
    MELT-1.6B & 0.089 & 2254.5 & 8.5 & 8.5 & 3.15 & 9.23 & 11 \\
    Ouro-1.4B-Thinking & 0.111 & 1795.7 & 13.5 & 13.5 & 2.86 & 26.58 & 3 \\
    Qwen3-1.7B & 0.029 & 6973.9 & 41.6 & 41.6 & 3.27 & 6.74 & 21 \\
    \bottomrule
  \end{tabular}%
  }
\end{table}

\subsection{Training efficiency}

Table~\ref{tab:training_efficiency} compares MELT with a 100-step Ouro distillation run under the same configuration. Chunk-wise training results in an approximately $10\times$ wall-clock slowdown.

\begin{table}[!ht]
  \centering
  \caption{Training efficiency metrics for MELT-1.6B and Ouro-1.4B.}
  \label{tab:training_efficiency}
  \begin{tabular}{lrrr}
    \toprule
    \textbf{Model} & \textbf{Tokens/s} & \textbf{Seconds/step} & \textbf{FLOPs/step} \\
    \midrule
    Ouro-1.4B & 20,645.2 & 31.00 & $1.862\times10^{15}$ \\
    MELT-1.6B & 2,168.1 & 295.19 & $2.236\times10^{15}$ \\
    \bottomrule
  \end{tabular}
\end{table}

\section{Notes on reproducibility and inference efficiency in Ouro}
\label{app:ouro_discussion}

This appendix documents several technical observations we made while attempting to reproduce and analyze claims from Zhu et al.~\citep{zhu2025scalinglatentreasoninglooped}. Our goal is not to diminish the contributions of Ouro, but to clarify practical aspects that are directly relevant when comparing memory usage and inference efficiency against MELT.

\subsection{Reproducibility}

% Despite substantial effort, we were unable to fully reproduce the reported results of Ouro under the configurations described in the paper. While the authors provide the model’s code and pretrained checkpoints, key implementation details and evaluation settings are either underspecified or differ from what is required to match the reported numbers. As a result, we observed non-trivial discrepancies between the performance reported in the original paper and the results obtained using the released artifacts. Therefore, throughout this work, we rely exclusively on values obtained from our own experimental measurements. We note that Gao et al.~\cite{gao_loop_2026} report similar findings when evaluating Ouro: their results are very close to ours, yet remain substantially below those reported in the original paper. Importantly, even under these independently measured evaluations, Ouro remains highly competitive and continues to outperform the state-of-the-art isoparametric vanilla baselines, underscoring the strength of its underlying approach.

Despite substantial effort, we were unable to fully reproduce the reported results of Ouro under the configurations described in the paper. While the authors provide the model’s code and pretrained checkpoints, key implementation details and evaluation settings are either underspecified or differ from what is required to match the reported numbers. As a result, we observed non-trivial discrepancies between the performance reported in the original paper and the results obtained using the released artifacts. Therefore, throughout this work, we rely exclusively on values obtained from our own experimental measurements. We note that Gao et al.~\cite{gao_loop_2026} report similar findings when evaluating Ouro: their results are very close to ours, yet remain substantially below those reported in the original paper. We additionally observe that, under our evaluation setup, Ouro-2.6B-Thinking does not outperform Qwen3-4B, contrary to the results reported in the original Ouro paper. Importantly, even under these independently measured evaluations, Ouro remains highly competitive and Ouro-1.4B-Thinking continues to outperform the state-of-the-art isoparametric vanilla baselines in our experiments, underscoring the strength of its underlying approach.

\subsection{Early-exit gating and effective compute}

A central component of Ouro is its learned gating mechanism, which is intended to enable adaptive computation by allowing tokens to exit early when additional recurrent steps are deemed unnecessary. While this mechanism is emphasized throughout the paper, we observed the following in practice:

\begin{itemize}
    \item Ouro introduces an early-exit mechanism, but the released default configuration uses a threshold that effectively disables early exiting, and the paper does not specify how this threshold should be chosen in practice.
    \item Even when an early exit is triggered, the model still executes all recurrent loop computations up to the maximum depth; the gating affects which logits are selected, not whether subsequent loops are computed.
\end{itemize}

We verified this behavior by inspecting the released inference code. As a consequence, the gating mechanism does not reduce inference-time compute or memory usage under typical settings, despite its conceptual framing as an adaptive compute mechanism.

We hypothesize that this design choice is driven by KV-cache dependencies: since later tokens require access to the KV states produced at the final loop, it is not possible to terminate computation early for a given token without breaking autoregressive consistency. It is worth noting that this limitation would not apply to MELT’s constant-memory KV update mechanism, although we leave a full investigation of early-exit strategies in MELT to future work.

% \subsection{KV cache sharing claims}

% In Section~5.4.2 of the Ouro paper (“KV Cache Sharing for Inference Efficiency”), the authors report up to a $4\times$ reduction in KV-cache memory usage. Although the paper notes that this KV-sharing strategy applies only to the generated tokens, while prompt tokens still use an unshared KV cache, the result requires careful interpretation. Since the optimization affects only part of the KV-cache, the effective end-to-end memory reduction in practical inference settings is substantially smaller than the stated $4\times$.

\section{Theoretical analysis}\label{app:theory}

As described in~\autoref{sec:method}, MELT updates only the last row of the state matrix. Therefore, we specialize our proofs to take that aspect explicitly into account. Notwithstanding, our results are fully generalizable to the case of the full matrix update. 

\subsection{Spectral stability of the gated update}

The next proposition shows that, in the saturated-gate regime, gradients are preserved across loops, providing the foundation for the Gradient Superhighway described in \autoref{sec:arch}.

This analysis is limited to dimensions in which the gate saturates near one and only characterizes the gated state update. It objective is not to establish stability of the full Transformer recurrence or bound non-saturated dimensions, attention, feed-forward layers, or normalization.

\begin{proposition}[Spectral Regulation in the Saturated Regime]
\label{prop:spectral_regulation}
Let $H_t^{(i)} \in \mathbb{R}^{L\times D}$ be the latent KV state at iteration $t$ for layer $i$, where $L$ is the sequence length and $D$ is the feature size. Let us consider $h_t$ as the $D\times 1$ state vector corresponding to the current token in the sequence (we omit the layer index for ease of notation). Consider MELT's element-wise update rule: 
\begin{equation}\label{eqn:had_prod}
    h_t = z_t \odot h_{t-1} + (\mathbf{1} - z_t) \odot x_t
\end{equation}
where $x_t \in \mathbb{R}^{D\times 1}$ is the hidden state at the current layer given input token $x$, and $z_t = \sigma(W_z x_t + U_z h_{t-1} + b_z) \in \mathbb{R}^{D\times 1}$ is the gate vector. If the gating mechanism saturates such that $z_t \to \mathbf{1}$ (element-wise) for a set of latent dimensions, then the Jacobian $J_t = \frac{\partial h_t}{\partial h_{t-1}}$, restricted to these dimensions, converges to the identity matrix. 
Consequently, the spectral radius $\rho(J_t) \to 1$. Thus, ensuring that the gradient magnitude is preserved over any arbitrary number of iterations.
\end{proposition}

\begin{proof}
The Jacobian of the update rule with respect to $h_{t-1}$ is derived via the product rule applied to the Hadamard product in~\autoref{eqn:had_prod}:
\begin{equation}
    J_t = \underbrace{\text{diag}(z_t)}_{\text{Term 1}} + \underbrace{\text{diag}(h_{t-1} - x_t) \frac{\partial z_t}{\partial h_{t-1}}}_{\text{Term 2}} + \underbrace{\text{diag}(\mathbf{1} - z_t) \frac{\partial x_t}{\partial h_{t-1}}}_{\text{Term 3}}
\end{equation}
We analyze the limit behavior in the shielding regime where $z_t \to \mathbf{1}$:
\begin{enumerate}
    \item \textbf{Term 1:} Approaches the identity matrix: $\mathrm{lim}_{z \to \mathbf{1}} \text{diag}(z_t) = \mathbf{I}$.
    \item \textbf{Term 2:} The derivative of the sigmoid function $\sigma'(u) = \sigma(u)(1-\sigma(u))$ vanishes as $z_t \to \mathbf{1}$. Thus, $\frac{\partial z_t}{\partial h_{t-1}} \to \mathbf{0}$.
    \item \textbf{Term 3:} The term $(\mathbf{1} - z_t)$ approaches $\mathbf{0}$, nullifying the contribution of the recurrent weight matrix in $\frac{\partial \tilde{h}_t}{\partial h_{t-1}}$.
\end{enumerate}
Consequently: $\mathrm{lim}_{z \to 1}J_t=\mathbf{I}+\mathbf{0}+\mathbf{0} \implies J_t \approx \mathbf{I}$. 
Since the eigenvalues of the identity matrix are all $1$, the spectral radius is $\rho(J_t) = 1$.
\end{proof}

Proposition~\ref{prop:spectral_regulation} gives more insights into the role of the gate $z_t$. Rather than simply selecting information, it acts as a structural stabilizer for the learning process. By explicitly controlling the decay rate of the hidden state, $z_t$ maintains the spectral radius of the recurrence dynamics near unity. This allows gradients to propagate through long sequences without vanishing, while the strict boundedness of the gate prevents the instability associated with exploding gradients.

\begin{remark}[Relaxation to Continuous Regime]
\label{rem:continuous_regime}
In practice, the gate $z_t$ is modeled via a sigmoid function and takes values in $[0, 1]$. 
While it cannot reach exactly $\mathbf{1}$, the network can learn weights such that $z_t \ge \mathbf{1} - \epsilon$ for an arbitrarily small $\epsilon > 0$. 
In this regime, the gradient magnitude at step $T$ scales as $(1 - \epsilon)^T$. 
Provided that $\epsilon \ll 1/T$, the signal degradation is negligible over the task's reasoning horizon. 
Thus, the continuous gate achieves \textit{Effective Spectral Regulation}, approximating the ideal stability derived in Proposition \ref{prop:spectral_regulation}.
\end{remark}

Our architecture establishes a retrieval hierarchy that decouples positional addressing from feature extraction. Because the state matrix $H$ maintains the history of processed tokens as discrete columns $\{h_1, \dots, h_L\}$, the model effectively possesses random access to the sequence axis. The attention mechanism first employs the Value projection $W_V$ to address the Disentanglement problem: since each token vector encodes multiple attributes in linear superposition, $W_V$ acts as a spectral filter, to isolate the specific feature subspace required for the current computation, ensuring that only the relevant signal is propagated while orthogonal interference is suppressed.

% Our architecture establishes a retrieval hierarchy that decouples positional addressing from feature extraction. Because the state matrix $H$ maintains the history of processed tokens as discrete columns $\{h_1, \dots, h_L\}$, the model effectively possesses random access to the sequence axis. The attention mechanism first employs the Key projection $W_K$ to resolve the Addressing problem, generating query-key scores that direct focus to the specific token position containing relevant context. Once a target column $h_n$ is selected, the Value projection $W_V$ addresses the Disentanglement problem: since each token vector encodes multiple attributes (syntax, semantics, position) in linear superposition, $W_V$ acts as a spectral filter—as detailed in Proposition~\ref{prop:disentanglement_matrix}—to isolate the specific feature subspace required for the current computation, ensuring that only the relevant signal is propagated while orthogonal interference is suppressed.

\subsection{Gradient Superhighway} \label{app:gradient_superhighway}

Stable gradient flow is essential for effectively optimizing the early loop iterations, so our architecture is designed to avoid both vanishing and exploding gradients.
To analyze this behavior, we examine how the loss $\mathcal{L}$ backpropagates to the initial state $h^{(l)}_{0}$ across the $T$ recurrent updates. By the chain rule, the gradient decomposes into a product of Jacobians:
\[
    \frac{\partial \mathcal{L}}{\partial h^{(l)}_0} = \frac{\partial \mathcal{L}}{\partial h^{(l)}_T} \prod_{t=1}^{T} \frac{\partial h^{(l)}_t}{\partial h^{(l)}_{t-1}} = \frac{\partial \mathcal{L}}{\partial h^{(l)}_T} \prod_{t=1}^{T} J_t
\]
Leveraging Proposition \ref{prop:spectral_regulation}, we observe that for latent dimensions in the shielding regime ($Z_t \approx \mathbf{1}$), the local Jacobian effectively acts as the identity operator ($J_t \approx \mathbf{1}$). This simplifies the product significantly:
\[
    \prod_{t=1}^{T} J_t \approx \prod_{t=1}^{T} \mathbf{1} = \mathbf{1}
\]
% (\autoref{fig:grad_flow})
Such a behavior establishes a \textit{Gradient Superhighway}: a direct path that allows error signals to traverse arbitrary depths. In contrast to standard recurrent dynamics, where gradient norms typically scale as $\mathcal{O}(\lambda^T)$—inevitably leading to exponential decay for spectral radii $|\lambda| < 1$—our architecture ensures that $\|\frac{\partial \mathcal{L}}{\partial h^{(l)}_0}\| \approx \|\frac{\partial \mathcal{L}}{\partial h^{(l)}_T}\|$. This structural stability effectively alleviates the vanishing gradient problem, enabling the optimization of deeper looped transformer models.

\section{Existing assets}
\label{app:assets}%% Tables listing datasets and models used in experiments.
%% Requires: \usepackage{booktabs}, \usepackage{hyperref}

In this appendix, we provide a comprehensive overview of all assets used throughout this work, along with their corresponding licenses to ensure transparency and reproducibility. Specifically, we list the evaluation benchmarks (Table~\ref{tab:datasets}), the models considered in our comparisons (Table~\ref{tab:models}), the training datasets employed (Table~\ref{tab:training_datasets}), and the main codebases used in our implementation (Table~\ref{tab:codebases}). All resources are referenced with links to their original sources.

% ----------------------------------------------------------------
% TABLE 1 — Datasets
% ----------------------------------------------------------------
\begin{table}[!ht]
  \centering
  \caption{The list of benchmarks.}
  \label{tab:datasets}
  \begin{tabular}{lcc}
    \toprule
    \textbf{Benchmarks} & \textbf{Link} & \textbf{License} \\
    \midrule
    MATH500~\citep{lightman2023lets}
& \href{https://huggingface.co/datasets/math-ai/math500}{\texttt{HuggingFace}}
& MIT \\[2pt]
    AIME 2024~\citep{aime2024}
& \href{https://huggingface.co/datasets/HuggingFaceH4/aime_2024}{\texttt{HuggingFace}}
& Copyright MAA \\[2pt]
    AIME 2025~\citep{aime2025}
& \href{https://huggingface.co/datasets/yentinglin/aime_2025}{\texttt{HuggingFace}}
& Copyright MAA \\[2pt]
    
AIME 2026~\citep{aime2026}
& \href{https://huggingface.co/datasets/math-ai/aime26}{\texttt{HuggingFace}}
& Apache 2.0 \\[2pt]
    OlympiadBench~\citep{he2024olympiadbench}
& \href{https://huggingface.co/datasets/math-ai/olympiadbench}{\texttt{HuggingFace}}
& MIT \\[2pt]
    AMC 2023~\citep{amc2023}
& \href{https://huggingface.co/datasets/math-ai/amc23}{\texttt{HuggingFace}}
& Copyright MAA \\[2pt]
    GPQA~\citep{rein2023gpqagraduatelevelgoogleproofqa}
& \href{https://huggingface.co/datasets/Idavidrein/gpqa}{\texttt{HuggingFace}}
& CC BY 4.0 \\
HLE~\citep{hle} & 
\href{https://huggingface.co/datasets/cais/hle}{\texttt{HuggingFace}} & MIT \\
MMLU~\citep{hendrycks2021measuringmassivemultitasklanguage} & \href{https://huggingface.co/datasets/cais/mmlu}{\texttt{HuggingFace}} & MIT \\
HumanEval~\citep{chen2021evaluatinglargelanguagemodels} & \href{https://huggingface.co/datasets/openai/openai_humaneval}{\texttt{HuggingFace}} & MIT \\
    \bottomrule
  \end{tabular}
\end{table}
 
% ----------------------------------------------------------------
% TABLE 2 — Models
% ----------------------------------------------------------------
\begin{table}[!ht]
  \centering
  \caption{The list of models.}
  \label{tab:models}
  %\small
  \begin{tabular}{lcc}
    \toprule
    \textbf{Model} & \textbf{Link} & \textbf{License} \\
    \midrule
    Ouro-Thinking-1.4B~\citep{zhu2025scalinglatentreasoninglooped}
& \href{https://huggingface.co/ByteDance/Ouro-1.4B-Thinking}{\texttt{HuggingFace}}
& Apache 2.0 \\[2pt]
    Ouro-Thinking-2.6B~\citep{zhu2025scalinglatentreasoninglooped}
& \href{https://huggingface.co/ByteDance/Ouro-2.6B-Thinking}{\texttt{HuggingFace}}
&  Apache 2.0 \\[2pt]
    Qwen3-1.7B~\citep{yang2025qwen3technicalreport}
& \href{https://huggingface.co/Qwen/Qwen3-1.7B}{\texttt{HuggingFace}}
& Apache 2.0 \\[2pt]
    Gemma4-E2B~\citep{gemma}
& \href{https://huggingface.co/google/gemma-4-E2B}{\texttt{HuggingFace}}
& Apache 2.0 \\[2pt]
    Qwen3.5-2B~\citep{qwen35blog}
& \href{https://huggingface.co/Qwen/Qwen3.5-2B}{\texttt{HuggingFace}}
& Apache 2.0 \\[2pt]
    DeepSeek-R1-1.5B~\citep{Guo_2025}
& \href{https://huggingface.co/deepseek-ai/DeepSeek-R1-Distill-Qwen-1.5B}{\texttt{HuggingFace}}
& MIT \\[2pt]
    \bottomrule
  \end{tabular}
\end{table}

\begin{table}[!ht]
  \centering
  \caption{The list of training datasets.}
  \label{tab:training_datasets}
  \begin{tabular}{lcc}
    \toprule
    \textbf{Dataset} & \textbf{Link} & \textbf{License} \\
    \midrule
    AceReason-1.1-SFT~\citep{liu_acereason-nemotron_2025}
& \href{https://huggingface.co/datasets/nvidia/AceReason-1.1-SFT}{\texttt{HuggingFace}}
& cc-by-4.0 \\[2pt]

    OpenThoughts3~\citep{guha2025openthoughtsdatarecipesreasoning}
& \href{https://huggingface.co/datasets/open-thoughts/OpenThoughts3-1.2M}{\texttt{HuggingFace}}
& Apache 2.0 \\[2pt]

%     OpenCodeReasoning~\citep{ahmad2025opencodereasoning}
% & \href{https://huggingface.co/datasets/nvidia/OpenCodeReasoning}{\texttt{HuggingFace}}
% & cc-by-4.0 \\[2pt]

%     Llama-Nemotron-PostTraining-Dataset~\citep{bercovich2025llamanemotronefficientreasoningmodels}
% & \href{https://huggingface.co/datasets/nvidia/Llama-Nemotron-Post-Training-Dataset}{\texttt{HuggingFace}}
% & cc-by-4.0 \\[2pt]

%     OO1-Chat-747K~\citep{oo1chat747k}
% & \href{https://huggingface.co/datasets/m-a-p/OO1-Chat-747K}{\texttt{HuggingFace}} \href{https://modelscope.cn/datasets/m-a-p/OO1-Chat-747K}{\texttt{ModelScope}}
% & Apache 2.0 \\
    \bottomrule
  \end{tabular}
\end{table}

\begin{table}[!ht]
  \centering
  \caption{The main codebases used.}
  \label{tab:codebases}
  \begin{tabular}{lcc}
    \toprule
    \textbf{Codebase} & \textbf{Link} & \textbf{License} \\
    \midrule
    Transformers~\citep{wolf-etal-2020-transformers}
& \href{https://github.com/huggingface/transformers}{\texttt{GitHub}}
& Apache 2.0 \\[2pt]

    TRL~\citep{vonwerra2020trl}
& \href{https://github.com/huggingface/trl}{\texttt{GitHub}}
& Apache 2.0 \\[2pt]

    vLLM~\citep{vllm}
& \href{https://github.com/vllm-project/vllm}{\texttt{GitHub}}
& Apache 2.0 \\[2pt]

    LightEval~\citep{lighteval}
& \href{https://github.com/huggingface/lighteval}{\texttt{GitHub}}
& MIT \\
    \bottomrule
  \end{tabular}
\end{table}

% \newpage
% \input{checklist.tex}

% \newpage

\end{document}